\documentclass{article} 
\pdfoutput=1
\usepackage{nips14submit_e,times}
\usepackage{hyperref}
\usepackage{url}
\usepackage{comment}
\usepackage{amsmath}

\usepackage{amsfonts}

\usepackage{graphicx}
\usepackage{multirow}
\usepackage{array}
\usepackage{floatrow}
\newfloatcommand{capbtabbox}{table}[][\FBwidth]

\def\for{\hbox{ for }}

\def\prob{\hbox{Pr}}
\def\Min{\text{Min}}
\def\Max{\text{Max}}

\def\bA{{\bf A}}
\def\bB{{\bf B}}
\def\bM{{\bf M}}
\def\bW{{\bf W}}
\def\bP{{\bf P}}
\def\RR{\mathbb R}
\usepackage{longtable}
\usepackage[titletoc]{appendix}
\usepackage{floatrow}
\usepackage{algorithm}
\usepackage{algorithmic}
\usepackage{framed}

\usepackage{cleveref}
\usepackage{xcolor}
\hypersetup{
    colorlinks,
    linkcolor={red!50!black},
    citecolor={blue!50!black},
    urlcolor={blue!80!black}
}

\def\citep{\cite}

\newtheorem{theorem}{Theorem}[section]
\newtheorem{lemma}[theorem]{Lemma}

\newtheorem{claim}{Claim}[section]

\newenvironment{proof}[1][Proof:]{\begin{trivlist}
\item[\hskip \labelsep {\bfseries #1}]}{\end{trivlist}}
\newenvironment{definition}[1][Definition:]{\begin{trivlist}
\item[\hskip \labelsep {\bfseries #1}]}{\end{trivlist}}

\newcommand{\qed}{\nobreak \ifvmode \relax \else
      \ifdim\lastskip<1.5em \hskip-\lastskip
      \hskip1.5em plus0em minus0.5em \fi \nobreak
      \vrule height0.4em width0.5em depth0.25em\fi}
\newcommand{\plmi}{$\pm\;$} 

\def\prob{\text{Prob}}

\title{ A provable SVD-based algorithm for learning topics in
dominant admixture corpus}

\author{
Trapit Bansal$\dagger$, C. Bhattacharyya$\ddagger$ \\
Department of Computer Science and Automation\\
Indian Institute of Science\\
Bangalore -560012, India \\
$\dagger$\texttt{trapitbansal@gmail.com} \\
$\ddagger$\texttt{chiru@csa.iisc.ernet.in} \\
\And
Ravindran Kannan \\
Microsoft Research \\
India\\
\texttt{kannan@microsoft.com} 
}

\nipsfinalcopy 

\begin{document}

\maketitle

\begin{abstract}
Topic models, such as Latent Dirichlet Allocation (LDA), posit that
documents are
drawn from admixtures of distributions over words, known as topics.
The inference problem of recovering topics from such a collection of documents drawn from admixtures,
is NP-hard.
Making a strong assumption called \emph{separability},
 \cite{AGM} gave the first provable algorithm
for inference. For the widely used LDA model, \cite{anandkumar-2} gave a provable algorithm using clever tensor-methods.
But \cite{AGM,anandkumar-2} do not learn topic vectors with bounded $l_1$ error (a natural measure for
probability vectors).

Our aim is to develop a model which makes intuitive and empirically supported
assumptions and to design an algorithm with natural, simple components such as SVD,
which provably solves the inference problem for the model with bounded $l_1$ error.
A topic in LDA and other models is essentially characterized by a group
of co-occurring words. Motivated by this, we introduce topic specific \emph{Catchwords}, a group of
words which occur with strictly  greater frequency in a topic than any other
topic individually and are required to have high frequency together rather than individually.
A major contribution of the paper is to show that
under this more realistic assumption, which is empirically verified on real corpora,
a singular value decomposition (SVD) based algorithm with a crucial pre-processing step
of thresholding,
can provably recover the topics from a collection of documents drawn from
\emph{Dominant admixtures}. Dominant admixtures are convex combination
of distributions in which one distribution has a significantly higher
contribution than the others. 
Apart from the simplicity of the algorithm, the
sample complexity has near optimal dependence on $w_0$,
the lowest probability that a topic is dominant, and is better
than \cite{AGM}. Empirical evidence
shows that on several real world corpora, both \emph{Catchwords} and
\emph{Dominant admixture} assumptions hold and the proposed algorithm
substantially outperforms the
state of the art \citep{arora}.
\end{abstract}

\section{Introduction}
Topic models \citep{survey}
assume that each document
 in a text corpus is generated from an \emph{ad-mixture}
of topics, where, each
topic is a distribution over words in a Vocabulary.
An admixture is a convex combination of distributions.
Words in the document are then picked in i.i.d.~trials, each trial
has a multinomial distribution over words given by the weighted combination of topic distributions. The problem of inference, recovering the topic
distributions
from such a collection of documents, is provably NP-hard.
Existing literature pursues techniques such as variational methods \citep{LDA} or MCMC procedures \citep{gibbs}
for approximating the maximum likelihood estimates.

Given the intractability of the problem one needs further assumptions
on topics to derive polynomial time algorithms which can
provably recover topics.
A possible (strong) assumption is that each document has only one topic but the collection can have many topics.
A document with only one topic is sometimes referred as a \emph{pure topic} document.
\citep{papa} proved that a natural algorithm, based on SVD,
recovers topics when each document is pure and in addition, for each topic, there is
a set of words, called \emph{primary words}, whose total frequency
in that topic is close to 1. More recently,
\citep{anandkumar-2}
show using tensor methods that if the topic weights
have Dirichlet distribution, we can learn the topic matrix. Note that
while this allows non-pure documents, the Dirichlet distribution gives
essentially uncorrelated topic weights.

In an interesting recent development  \cite{AGM, arora} gave the first
provable algorithm which can recover topics from a corpus of documents
drawn from admixtures, assuming \emph{separability}. Topics are said to be separable
if in every topic there exists at least one \emph{Anchor} word.
A word in a topic is said to be an an \emph{Anchor} word for that
topic if it has a high probability in that topic and
\emph{zero} probability in remaining topics. The requirement of high
probability in a topic for a single word is unrealistic.

\paragraph{Our Contributions:}
Topic distributions, such as those learnt in LDA, try to model the co-occurrence of a group
of words which describes a theme. Keeping this in mind we introduce
the notion of \emph{Catchwords}. A group of words are called
\emph{Catchwords} of a topic, if each word occurs strictly more frequently in the topic than other
topics and together they have high frequency.
This is a much weaker assumption than
separability.
Furthermore we observe, empirically, that posterior topic weights assigned by LDA
to a document often have the property that
one of the weights is significantly higher than the rest.
Motivated by this observation, which has not been exploited by topic modeling
literature, we suggest a new assumption.
It is natural to assume that in a text corpus,
a document, even if it has multiple themes, will have an
overarching dominant theme. In this paper we focus
on document collections drawn from \emph{dominant admixtures}.
A document collection is said to be drawn from a dominant admixture if
for every document, there is one topic whose weight is
significantly higher than other topics and in addition, for every topic,
there is a small fraction of documents
which are nearly purely on that topic.
The main contribution of the paper is to show that
under these assumptions, our algorithm, which we call TSVD ,
indeed provably finds a good approximation in total $l_1$ error to
the topic matrix. We prove a bound on the error of our approximation which does not
grow with dictionary size $d$, unlike \cite{arora} where the error grows
linearly with $d$.

Empirical evidence shows that on semi-synthetic
corpora constructed from several real world datasets, as suggested by \cite{arora}, TSVD
substantially outperforms
the state of the art \citep{arora}. In particular it is seen that
compared to \cite{arora}
TSVD gives $27$\% lower error in terms of $L_1$ recovery on $90$\%
of the topics.

\paragraph{Problem Definition:}
$d,k,s$ will denote respectively, the number of words in the dictionary, number of topics and
number of documents. $d,s$ are large, whereas, $k$ is to be thought of as much smaller.
Let ${\cal S}_k=\{ x=(x_1,x_2,\ldots ,x_k): x_l\geq 0; \sum_l x_l=1\}$. For each topic, there
is a fixed vector in ${\cal S}_k$ giving the probability of each word in that topic. Let $\bM$
be the $d\times k$ matrix with these vectors as its columns.

Documents are picked in i.i.d. trials. To pick document $j$, one first picks a $k-$ vector
$W_{ij},W_{2j},\ldots , W_{kj}$ of topic weights according to
a fixed distribution on ${\cal S}_k$. Let
$P_{\cdot, j}= \bM W_{\cdot,j}$ be the weighted combination of the topic vectors. Then the $m$
words of the document are picked in i.i.d. trials; each trial picks a word according to the
multinomial distribution with $P_{\cdot,j}$ as the probabilities. All that is given as
data is the frequency of words in each document, namely, we are given the $d\times s$ matrix
$\bA$, where~ $A_{ij}=\frac{\text{Number of occurrences of word $i$ in Document $j$}}{m}.$
Note that $E(\bA | \bW) = \bP $, where, the expectation is taken entry-wise.

\emph{In this paper we consider the problem of finding $\bM$ given $\bA$.}

\section{Previous Results}\label{sec:prev}
In this section we review literature related to designing provable algorithms
for topic models. For an overview of topic models we refer the reader to
the excellent survey\citep{survey}.
Provable algorithms for recovering topic models was started by \cite{papa}.
Latent Semantic Indexing(LSI) \citep{lsi} remains a successful method for retrieving
similar documents by using SVD. \cite{papa} showed that one can recover $\bM$ from a collection of documents,
with \emph{pure topics}, by using SVD based procedure under the additional Primary Words Assumption.
%
 \citep{anandkumar-2} showed that in the admixture case, if one assumes Dirichlet distribution
for the topic weights, then, indeed, using tensor methods, one can learn $\bM$ to $l_2$ error provided some
added assumptions on numerical parameters like condition number are satisfied. 

The first provably polynomial time algorithm  for admixture corpus
 was given in \citep{AGM,arora}.
For
a  topic $l$, a word $i$ is an {\bf anchor word} if
$M_{i,l}\geq p_0\quad ;\quad M_{i,l'}=0\quad\forall l'\not= l.$
\begin{theorem}\citep{AGM}
If every topic has an anchor word,
there is a polynomial time algorithm that returns an $\hat M$ such that with high probability,
$$\sum_{l=1}^k\sum_{i=1}^d |\hat M_{il}-M_{il}|\leq \; d\varepsilon\; \;  \text{    provided   }\;
s\geq \text{Max}\left\{ O\left( \frac{ k^6\log d}{a^4\varepsilon ^2p_0^6\gamma^2m}\right), O\left( \frac{k^4}{\gamma^2a^2}\right)\right\},$$
where, $\gamma$ is the condition number of $E(WW^T)$, $a$ is the minimum expected weight
of a topic and $m$ is the number of words in each document.
\end{theorem}
Note that the error grows linearly in the dictionary size $d$, which is often large.
Note also the dependence of $s$ on parameters $p_0$, which is, $1/p_0^6$ and on $a$, which is $1/a^4$.
If, say, the word ``run'' is an anchor word for the topic ``baseball'' and $p_0=0.1$, then
the requirement is that every 10 th word in a document on this topic is ``run''. This seems
too strong to be realistic. It would be more realistic to ask that a set of words like - ``run'', ``hit'', ``score'',
etc. together have frequency at least 0.1 which is what our catchwords assumption does.

\section{Learning Topics from Dominant Admixtures}
Informally, a document is said to be drawn
from a Dominant Admixture if the document has one \emph{dominant topic}.
Besides its simplicity, we show
empirical evidence from real corpora to demonstrate that topic dominance is a reasonable
assumption.
The Dominant Topic assumption is
 weaker than the Pure Topic assumption.
More importantly
SVD based procedures proposed by \citep{papa} will not apply.
Inspired by the \emph{Primary words} assumption we introduce the
assumption that each topic has a set
of \emph{Catchwords} which
individually occur more frequently in that topic than others. This is
again a much weaker assumption than both \emph{Primary Words} and \emph{Anchor
Words} assumptions and can be verified experimentally.
 In this section we establish that
by applying SVD on a matrix, obtained by thresholding the word-document matrix,
and subsequent $k$ means clustering can learn topics having Catchwords
from a Dominant Admixture corpus.

\subsection{Assumptions: Catchwords and Dominant admixtures}
\label{sec:assumptions}
Let
$\alpha,\beta,\rho,\delta,\varepsilon_0 $ be non-negative reals satisfying:
\begin{align}
\beta+\rho &\leq (1-\delta)\alpha\label{alpha-beta-rho}.\\
\alpha+2\delta &\leq 0.5\quad ;\quad \delta\leq 0.08\label{delta-inequality}.
\end{align}

{\bf Dominant topic Assumption} (a) For $j=1,2,\ldots ,s$, document $j$ has a dominant topic
$l(j)$ such that
$W_{l(j),j}\geq\alpha \text{  and  }  W_{l'j}\leq\beta,\;\; \forall l'\not= l(j).$

(b)For each topic $l$, there are at least $\varepsilon_0 w_0s$ documents in each of which topic $l$ has weight at least $1-\delta$.

{\bf Catchwords Assumption:}
There are $k$ disjoint sets of words - $S_1,S_2,\ldots ,S_k$
such that with $\varepsilon$ defined in (\ref{varepsilon-inequality})
\begin{align}
& \forall i\in S_l,\; \forall l'\not= l, \; M_{il'}\leq\rho M_{il}\label{401}\\
& \sum_{i\in S_l}M_{il}\geq p_0\label{402}\\
&\forall i\in S_l,  m\delta^2\alpha M_{il}\geq 8\ln\left( \frac{20}{\varepsilon w_0}\right)\label{def:catch}.
\end{align}

Part (b.) of the Dominant Topic Assumption is in a sense necessary for ``identifiability'' - namely for the model to
have a set of $k$ document vectors so that every document vector is in the convex hull of these
vectors.
The Catchwords assumption is natural to describe a theme
as it tries to model
a unique group of words which is likely to co-occur when a theme is expressed. This assumption is close to topics discovered by
LDA like models, which try to model of co-occurence of words.
If $\alpha,\delta\in\Omega(1)$, then, the assumption (\ref{def:catch}) says $M_{il}\in\Omega^*(1/m)$. In fact if
$M_{il}\in o(1/m)$, we do not expect to see word $i$ (in topic $l$), so it cannot be called a catchword at all.

A slightly different (but equivalent) description of the model will be useful to keep in mind.
What is fixed (not stochastic) are the matrices $\bM$ and the distribution of the weight matrix $\bW$.
To pick document $j$, we can first pick the dominant topic $l$ in document $j$
and condition the
distribution of $W_{\cdot ,j}$ on this being the dominant topic. One could instead also think of $W_{\cdot ,j}$
being picked from a mixture of $k$ distributions.
Then, we
let $P_{ij}=\sum_{l=1}^k M_{il}W_{lj}$ and
pick the $m$ words of the document in i.i.d multinomial trials as before.
We will assume that
$$T_l=\{ j: l\text{  is the dominant topic in document }j\} \text{  satisfies  }     |T_l|=w_ls,$$
where, $w_l$ is the probability of topic $l$ being dominant.
This is only approximately valid, but the error is small enough that we can disregard it.

For $\zeta\in \{ 0,1,2,\ldots ,m\}$, let $p_i(\zeta,l)$ be the probability that $j\in T_l$ and $A_{ij}=\zeta/m$ and $q_i(\zeta,l)$ the corresponding ``empirical probability'':
\begin{align}
p_i(\zeta,l) &= \int_{W_{\cdot, j}} {m\choose \zeta} P_{ij}^\zeta (1-P_{ij})^{m-\zeta}  \prob(W_{\cdot, j}\; |\; j\in T_l) \; \prob(j\in T_l), \text{  where,  }P_{\cdot,j}=\bM W_{\cdot,j}. \label{p-i-zeta-l}\\
q_i(\zeta,l)&=\frac{1}{s}\left| \{ j\in T_l: A_{ij}=\zeta/m\}\right| \label{q-i-zeta-l}.
\end{align}
Note that $p_i(\zeta,l)$ is a real number, whereas, $q_i(\zeta,l)$ is a random variable with
$E(q_i(\zeta,l))=p_i(\zeta,l)$.
We need a technical assumption on the $p_i(\zeta,l)$ (which is weaker than unimodality).

{\bf No-Local-Min Assumption:}
We assume that $p_i(\zeta,l)$ does not have a local
minimum, in the sense:
\begin{equation}\label{fil}
p_i(\zeta,l) \; > \; \Min (p_i(\zeta-1,l),p_i(\zeta+1,l))\; \forall \; \zeta \in \{ 1,2,\ldots ,m-1\}.
\end{equation}

The justification for the this assumption is two-fold. First, generally, Zipf's law
kind of behaviour where the number of words plotted against relative frequencies declines
as a power function has often been observed. Such a plot is monotonically decreasing and
indeed satisfies our assumption. But for Catchwords, we do not expect this behaviour - indeed,
we expect the curve to go up initially as the relative frequency increases, then reach a maximum
and then decline. This is a unimodal function and also satisfies our assumption. Indeed, we
have empirically observed, see {\tt EXPTS}, that these are essentially the only two behaviours.

{\bf Relative sizes of parameters}
Before we close the section we discuss the values of the parameters are
in order.
Here, $s$ is large. For asymptotic analysis, we can
think of it as going to infinity. $1/w_0$ is also large and can be thought of as going
to infinity. [In fact, if $1/w_0\in O(1)$, then, intuitively, we see that there is no
use of a corpus of more than constant size - since our model has i.i.d. documents, intuitively,
the number of samples we need should depend mainly on $1/w_0$].
$m$ is (much) smaller, but need not be constant.

$c$ refers to a generic constant independent of $m,s,1/w_0,\varepsilon,
\delta$; its value may be different in different contexts.

\subsection{The TSVD Algorithm}
\label{sec:algo}
Existing SVD based procedures for clustering on raw word-document matrices
fail because
the spread of frequencies of a word within a topic is often more (at least not significantly less)
than the gap between the word's frequencies in two different topics.
Hypothetically the frequency for the word  \emph{run},
in the topic \emph{Sports}, may range from say 0.01 on up. But in other
topics, it may range from 0 to 0.005 say.
The success of the algorithm will lie on correctly identifying the dominant topics such as sports by identifying that the word \emph{run} has occurred
with high frequency. In this example, the gap (0.01-0.005) between Sports and other topics is less than the spread within
Sports (1.0-0.01), so a 2-clustering approach (based on SVD) will split the topic Sports into two. While this is a toy
example, note that if we threshold the frequencies at say 0.01, ideally, sports will be all above and the
rest all below the threshold, making the succeeding job of clustering easy.

There are several issues in
extending beyond the toy case. Data is not one-dimensional. We will use different thresholds for each
word; word $i$ will have a threshold $\zeta_i/m$. Also, we have to compute $\zeta_i/m$.
Ideally we would not like to split any $T_l$, namely, we would like that
for each $l$ and and each $i$, either most $j\in T_l$ have $A_{ij}>\zeta_i/m$ or most
$j\in T_l$ have $A_{ij}\leq\zeta_i/m$. We will show that our threshold procedure indeed achieves this.
One other nuance: to avoid conditioning, we split the data $\bA$ into two parts ${\bf A^{(1)}}$ and
${\bf A^{(2)}}$, compute the thresholds using ${\bf A^{(1)}}$ and actually do the thresholding on
${\bf A^{(2)}}$. We will assume that the intial $\bA$ had $2s$ columns, so each part now has $s$ columns.
Also, $T_1,T_2,\ldots ,T_k$ partitions the columns of $\bA^{(1)}$ as well as those of $\bA^{(2)}$.
The columns of thresholded matrix $\bB$ are then clustered, by a technique we call
Project and Cluster, namely, we project the columns of $\bB$ to its $k-$dimensional SVD
subspace and cluster in the projection. The projection before clustering has recently been
proven \citep{KK} (see also \citep{AS}) to yield good starting cluster centers. The clustering so found is not yet
satisfactory. We use the classic Lloyd's $k$-means
algorithm proposed by \cite{Lloyd82}. As we will show, the partition produced after clustering,
$\{R_1,\ldots,R_k\}$ of $\bA^{(2)}$
is close to the partition induced by the Dominant Topics, $\{T_1,\ldots,T_k\}$.
Catchwords of topic $l$ are now (approximately) identified as the most frequently occurring words in documents
in $R_l$. Finally, we identify nearly pure documents in $T_l$ (approximately) as the documents
in which the catchwords occur the most. Then we get an approximation to $M_{\cdot, l}$ by averaging these
nearly pure documents. We now describe the precise algorithm.

\subsection{Topic recovery using Thresholded SVD}

{\bf Threshold SVD based K-means (TSVD) }
\begin{equation}\label{varepsilon-inequality}
\varepsilon  = \text{Min }\left( \frac{1}{900c_0^2}\frac{\alpha p_0}{k^3m}\; ,\; \frac{\varepsilon _0 \sqrt{\alpha p_0}\delta }{640m\sqrt k}\; ,\;
\right).
\end{equation}

\begin{enumerate}
    \item  Randomly partition the columns of $\bA$ into two
matrices $\bA^{(1)}$ and $\bA^{(2)}$ of $s$~columns each.
    \item {\bf Thresholding}
    \begin{enumerate}
         \item {\bf Compute Thresholds on $\bA^{(1)}$} For each $i$, let
            $\zeta_i$ be the highest value of $\zeta\in \{ 0,1,2,\ldots ,m\}$ such that $|\{ j: A^{(1)}_{ij}> \frac{\zeta}{m}\}|\geq \frac{w_0}{2} s$;~
$ |\{ j: A^{(1)}_{ij} =\frac{\zeta}{m}\}| \leq  3\varepsilon w_0s.$
        \item {\bf Do the thresholding on }$\bA^{(2)}$:~
            $ B_{ij}=\begin{cases} \sqrt {\zeta_i} & \mbox{ if } A^{(2)}_{ij} > \zeta_i/m \mbox{ and }\zeta_i\geq 8\ln (20/\varepsilon w_0)\\
                    0 &\mbox{ otherwise  } \end{cases}.$
    \end{enumerate}
    \item {\bf SVD} Find the best rank $k$ approximation $\bB^{(k)}$ to $\bB$.
    \item {\bf Identify Dominant Topics}
    \begin{enumerate}
       \item {\bf Project and  Cluster} Find (approximately) optimal $k-$ means clustering of
        the columns of $\bB^{(k)}$.
        \item {\bf Lloyd's Algorithm} Using the clustering found in Step 4(a) as the starting clustering,
            apply Lloyd's algorithm $k$ means algorithm to the columns of $\bB$ ($\bB$, not $\bB^{(k)}$).
        \item   Let $R_1,R_2,\ldots ,R_k$ be the $k-$partition of $[s]$
        corresponding to the clustering after Lloyd's. //*\emph{Will prove that $R_l\approx T_l$}*//
    \end{enumerate}
    \item {\bf Identify Catchwords}
   \begin{enumerate}
        \item For each $i,l$, compute $g(i,l)=$ the $\lfloor \varepsilon_0w_0s/2 \rfloor)$
        th highest element of $\{ A^{(2)}_{ij} : j\in R_l\}$.
        \item Let $J_l=\left\{ i: g(i,l)>\Max\left(\frac{4}{m\delta^2}\ln (20/\varepsilon w_0),\Max_{l'\not= l}\gamma\; g(i,l')\right)\right\},$
        where, $\gamma = \frac{1-2\delta}{(1+\delta)(\beta+\rho)}$.
   \end{enumerate}
   \item {\bf Find Topic Vectors}  Find the $\lfloor \varepsilon_0w_0s /2\rfloor $ highest $\sum _{i\in J_l}A_{ij}^{(2)}$ among all $j\in [s]$ and return the average
   of these $A_{\cdot,j}$
   as our approximation $\hat M_{\cdot, l}$ to $M_{\cdot,l}$.
\end{enumerate}

\begin{theorem}{\bf Main Theorem}\label{main-theorem}
Under the Dominant Topic, Catchwords and No-Local-Min assumptions, the algorithm
succeeds with high probability in finding an $\hat M$ so that
$$\sum_{i,l}|M_{il}-\hat M_{il}|\in O(k\delta), \text{   provided   }\footnote{The superscript $^*$ hides a logarithmic factor in $dsk/\delta_{\text{fail}}$, where,
$\delta_{\text{fail}}>0$ is the desired upper bound on the probability of failure.}
s\in \Omega^*\left( \frac{1}{w_0}\left( \frac{k^6m^2}{\alpha^2p_0^2}+\frac{m^2k}{\varepsilon_0^2\delta^2\alpha p_0}
 +\frac{d}{\varepsilon_0\delta^2}\right)\right).$$
\end{theorem}

A note on the sample complexity is in order. Notably, the dependence of $s$ on $w_0$ is best possible
(namely $s\in\Omega^*(1/w_0)$) within logarithmic factors, since, if we had fewer than $1/w_0$ documents,
a topic which is dominant with probability only $w_0$ may have none of the documents in the collection.
The dependence of $s$ on $d$ needs to be at least $d/\varepsilon_0w_0\delta^2$: to see this,
note that we only assume that there are $r=O(\varepsilon _0w_0s)$ nearly pure documents on each topic.
Assuming we can find this set (the algorithm approximately does), their average has standard deviation of about
$\sqrt{M_{il}}/\sqrt r$ in coordinate $i$.
If
topic vector $M_{\cdot, l}$ has $O(d)$ entries, each of size $O(1/d)$,
to get an approximation of
$M_{\cdot,l}$ to $l_1$ error $\delta$, we need the per coordinate error $1/\sqrt {dr}$ to be at most $\delta/d$ which implies $s\geq d/\varepsilon_0w_0\delta^2$.
Note that to get comparable error in \cite{AGM}, we need a quadratic dependence on $d$.

There is a long sequence of Lemmas to prove the theorem. The Lemmas and the proofs are given in Appendix.
The essence of the proof lies in proving that the clustering
step correctly identifies the partition induced by the dominant topics.
For this, we take advantage of a recent
development
on the $k-$means algorithm from
\cite{KK}
[see also \cite{AS}], where, it is shown that under a condition called
the \emph{Proximity Condition},
Lloyd's $k$ means algorithm starting with the centers provided by the SVD-based algorithm,
correctly identifies almost all the documents' dominant topics.  We prove that indeed the Proximity Condition
holds. This calls for machinery from Random Matrix theory (in particular bounds on singular values). We
prove that the singular values of the thresholded word-document matrix are nicely bounded.
Once the dominant topic of each document is identified, we are able to find the Catchwords for each topic.
Now, we rely upon part (b.) of the Dominant Topic assumption : that is there is a small fraction of nearly Pure
Topic-documents for each topic. The Catchwords help isolate the nearly pure-topic documents and hence find the
topic vectors. The proofs are complicated by the fact that each step of the algorithm induces conditioning on the data-
for example, after clustering, the document vectors in one cluster are not anymore independent.

\section{Experimental Results}\label{sec:expts}
We compare the thresholded SVD based k-means (\textbf{TSVD}\footnote{Resources available at: \url{http://mllab.csa.iisc.ernet.in/tsvd}}) algorithm \ref{sec:algo} with the algorithms of \cite{arora}, \emph{Recover-KL} and \emph{Recover-L2}, using the code made available by the authors\footnote{\url{http://www.cs.nyu.edu/~halpern/files/anchor-word-recovery.zip}}.
We first provide empirical support for the algorithm assumptions in Section \ref{sec:assumptions}, namely the dominant topic and the catchwords assumption.
Then we show on 4 different semi-synthetic data that TSVD provides as good or better recovery of topics than the Recover algorithms.
Finally on real-life datasets, we show that the algorithm performs as well as \cite{arora} in terms of perplexity and topic coherence.

\paragraph{Implementation Details:}
TSVD parameters ($w_0,\;\varepsilon,\;\varepsilon_0,\;\gamma$) are not known in advance for real corpus. We tested empirically for multiple settings and the following values gave the best performance.
Thresholding parameters used were: $w_0=\frac{1}{k}$, $\varepsilon = \frac{1}{6}$. 
For finding the catchwords, $\gamma = 1.1, \varepsilon_0=\frac{1}{3}$ in step 5.
For finding the topic vectors (step 6), taking the top 50\% ($\varepsilon_0 w_0= \frac{1}{k}$) gave empirically better results.
The same values were used on all the datasets tested.
The new algorithm is sensitive to the initialization of the first k-means step in the projected SVD space. 
To remedy this, we run 10 independent random initializations of the algorithm with K-Means++ \cite{kmeanspp} and report the best result.

{\bf Datasets: }
We use four real word datasets in the experiments.
As pre-processing steps we removed standard stop-words, selected the vocabulary size by term-frequency and removed documents with less than 20 words. Datasets used are:
 (1) {\bf NIPS}\footnotemark[4]\footnotetext[4]{\url{http://archive.ics.uci.edu/ml/datasets/Bag+of+Words}}: Consists of 1,500 NIPS full papers, vocabulary of 2,000 words and mean document length 1023. 
(2) {\bf NYT}\footnotemark[4]: Consists of a random subset of 30,000 documents from the New York Times dataset, vocabulary of 5,000 words and mean document length 238.
(3) {\bf  Pubmed}\footnotemark[4]: Consists of a random subset of 30,000 documents from the Pubmed abstracts dataset, vocabulary of 5,030 words and mean document length 58.
(4) {\bf 20NewsGroup}\footnotemark[5]\footnotetext[5]{\url{http://qwone.com/~jason/20Newsgroups}} (20NG): Consist of 13,389 documents, vocabulary of 7,118 words and mean document length 160.

\subsection{Algorithm Assumptions}
\label{sec:assumption}
To check the \emph{dominant topic} and \emph{catchwords} assumptions, we first run 1000 iterations of Gibbs sampling on the real corpus and learn the posterior document-topic distribution ($\{W_{.,j}\}$) for each document in the corpus
(by averaging over 10 saved-states separated by 50 iterations after the 500 burn-in iterations).
We will use this posterior document-topic distribution as the document generating distribution to check the two assumptions.

\textbf{Dominant topic assumption: }
Table \ref{tab:assumptions} shows the fraction of the documents in each corpus which satisfy this assumption with $\alpha=0.4$ (minimum probability of dominant topic) and $\beta=0.3$ (maximum probability of non-dominant topics).
The fraction of documents which have almost pure topics with highest topic weight at least 0.95 ($\delta = 0.05$) is also shown. 
The results indicate that the dominant topic assumption is well justified (on average 64\% documents satisfy the assumption) 
and there is also a substantial fraction of documents satisfying almost pure topic assumption.

\textbf{Catchwords assumption: }
We first find a $k$-clustering of the documents $\{T_1,\ldots,T_k\}$ by assigning all documents which have highest posterior probability for the same topic into one cluster.
Then we use step 5 of TSVD (Algorithm \ref{sec:algo}) to find the set of catchwords for each topic-cluster, i.e. $\{S_1,\ldots,S_k\}$, 
with the parameters:
$\epsilon_0 w_0=\frac{1}{3k}$, $\gamma = 2.3$ (taking into account constraints in Section \ref{sec:assumptions}, $\alpha = 0.4, \beta=0.3, \delta=0.05, \rho = 0.07$).
Table \ref{tab:assumptions} reports the fraction of topics with non-empty set of catchwords and the average per topic frequency of the catchwords\footnotemark[6]\footnotetext[6]{$\left( \frac{1}{k}\sum_{l=1}^{k} \frac{1}{|T_l|}\sum_{i\in S_l} \sum_{j\in T_l} A_{ij} \right)$}.
Results indicate that most topics on real data contain catchwords (Table \ref{tab:assumptions}, second-last column).
Moreover, the average per-topic frequency of the group of catchwords for that topic is also quite high (Table \ref{tab:assumptions}, last column).

\textbf{No-Local-Min Assumption:}
To provide support and intuition for the local-min assumption we consider the quantity $q_i(\zeta,l)$, in (\ref{q-i-zeta-l}).
Recall that $\mathbb{E}[q_i(\zeta,l)] = p_i(\zeta,l)$, we will analyze the behavior of $q_i(\zeta,l)$ as a function of $\zeta$ for some topics $l$ and words $i$.
As defined, we need a fixed $m$ to check this assumption and so we generate semi-synthetic data with a fixed $m$ from LDA model trained on the real NYT corpus (as explained in Section \ref{sec:l1}).
We find catchwords and the sets $\{T_l\}$ as in the catchwords assumption above and plot $q_i(\zeta,l)$ separately for some random catchwords and non-catchwords by fixing some random $l \in [k]$.
Figure \ref{fig:localmin} shows the plots. As explained in \ref{sec:assumptions}, the plots are monotonically decreasing for non-catchwords and satisfy the assumption.
On the other hand, the plots for catchwords are almost unimodal and also satisfy the assumption.

\begin{table}[t!]
 \centering
 \begin{tabular}{|c|c|c||c|c|c|c|}
 \hline
  \multirow{2}{*}{\bf Corpus} & \multirow{2}{*}{$s$} & \multirow{2}{*}{$k$} & {\bf \% s with Dominant} & {\bf \% s with Pure} & {\bf \% Topics} & {\bf CW Mean} \\ 
   & & & {\bf Topics ($\alpha=0.4$)} & {\bf Topics ($\delta=0.05$)} & {\bf with CW} & {\bf Frequency}   \\ \hline
   NIPS & 1500 & 50 & 56.6\% & 2.3\% & 96\% & 0.05 \\ \hline
   NYT & 30000 & 50 & 63.7\% & 8.5\% & 98\% & 0.07 \\ \hline
   Pubmed & 30000 & 50 & 62.2\% & 5.1\% & 78\% & 0.05  \\ \hline
   20NG & 13389 & 20 & 74.1\% & 39.5\% & 85\% & 0.06 \\ \hline
 \end{tabular}
\caption{Algorithm Assumptions. For dominant topic assumption, fraction of documents with satisfy the assumption for $(\alpha,\beta)=(0.4,0.3)$ are shown.
\% documents with almost pure topics ($\delta=0.05$, i.e. $95\%$ pure) are also shown.
Last two columns show results for catchwords (CW) assumption.}
\label{tab:assumptions}
\end{table}

\begin{figure}[t!]
 \includegraphics[width=\textwidth]{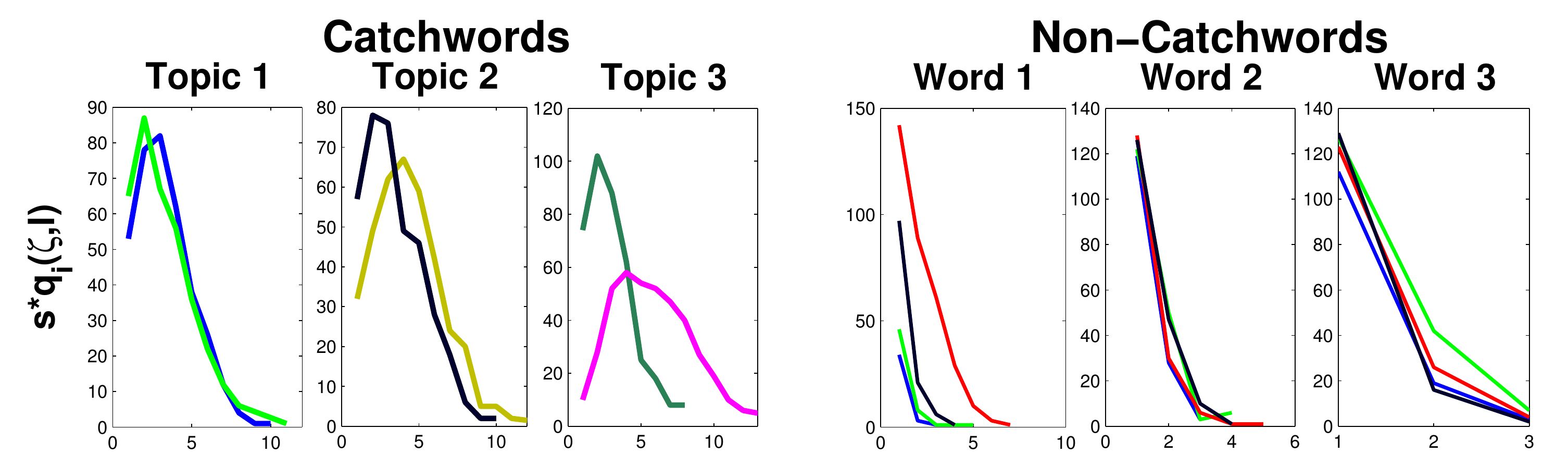}
 \caption{Plot of $q_i(\zeta,l)$ for some random catchwords (left) and non-catchwords (right). Each of three plots for catchword is for one topic ($l$) with two random catchwords ($i$) for each topic and each plot on right is for one non-catchword ($i$) with curves for multiple topics ($l$).}
 \label{fig:localmin}
\end{figure}

\subsection{Empirical Results}
\subsubsection{Topic Recovery on Semi-Synthetic Data}
\label{sec:l1}
\textbf{Semi-synthetic Data: }
Following \cite{arora}, we generate semi-synthetic corpora from LDA model trained by MCMC, to ensure that the synthetic corpora retain the characteristics of real data.
Gibbs sampling is run for 1000 iterations on all the four datasets and the final word-topic distribution is used to generate varying number of synthetic documents with document-topic distribution drawn from a symmetric Dirichlet with hyper-parameter 0.01.
For NIPS, NYT and Pubmed we use $k=50$ topics, for 20NewsGroup $k=20$, and mean document lengths of 1000, 300, 100 and 200 respectively.
Note that the synthetic data is \emph{not} guaranteed to satisfy dominant topic assumption for every document (on average about 80\% documents satisfy the assumption for value of $(\alpha,\beta)$ tested in Section \ref{sec:assumption})

\textbf{Topic Recovery:} 
We learn the word-topic distributions ($\hat{M}$) for the semi-synthetic corpora using TSVD and the Recover algorithms of \cite{arora}.
Given these learned topic distributions and the original data-generating distributions ($M$), we align the topics of $M$ and $\hat{M}$ by bipartite matching 
and rearrange the columns of $\hat{M}$ in accordance to the matching with $M$.
Topic recovery is measured by the average of the $l_1$ error across topics (called reconstruction error \citep{arora}), $\Delta(M,\hat{M})$, defined as:
 $\Delta(M,\hat{M}) = \frac{1}{k} \sum_{l=1}^k \sum_{i=1}^{d} |M_{il} - \hat{M}_{il}|$.

We report reconstruction error in Table \ref{tab:l1_all} for TSVD and the Recover algorithms, Recover-L2 and Recover-KL.
TSVD has smaller error on most datasets than the Recover-KL algorithm. We observed performance of TSVD to be always better than Recover-L2.
Best performance is observed on NIPS which has largest mean document length, indicating that larger $m$ leads to better recovery.
Results on 20NG are slightly worse than Recover-KL for small sample size (though better than Recover-L2), but the difference is small.
While the values in Table \ref{tab:l1_all} are averaged values,
Figure \ref{fig:hist_all} shows that TSVD algorithm achieves much better topic recovery (27\% improvement in $l_1$ error over Recover-KL)
for majority of the topics ($>$90\%) on most datasets.

\begin{table}[t!]
\centering 
\begin{tabular}{|>{\bfseries}c||c||c|c|c|c|}
\hline 
 Corpus & \textbf{Documents} & \textbf{Recover-L2} & \textbf{Recover-KL} & \textbf{TSVD} & \textbf{\% Improvement}
 \\[2pt] \hline \hline
 \multirow{4}{*}{NIPS} 
 & 40,000 & 0.342 & 0.308 & \textbf{0.115} & \textbf{62.7\%}
\\ \cline {2-6}
 & 50,000 & 0.346 & 0.308 & \textbf{0.145} & \textbf{52.9\%}
\\ \cline {2-6}
 & 60,000 & 0.346 & 0.311 & \textbf{0.131} & \textbf{57.9\%}
\\ \hline \hline

 \multirow{4}{*}{Pubmed} 
 & 40,000 & 0.388 & 0.332 & \textbf{0.288} & \textbf{13.3\%}
\\ \cline {2-6}
 & 50,000 & 0.378 & 0.326 & \textbf{0.280} & \textbf{14.1\%}
\\ \cline {2-6}
 & 60,000 & 0.372 & 0.328 & \textbf{0.284} & \textbf{13.4\%}
\\ \hline \hline

 \multirow{4}{*}{20NG} 
 & 40,000 & 0.126 & \textbf{0.120} & 0.124 & -3.3\%
\\ \cline {2-6}
 & 50,000 & 0.118 & 0.114 & \textbf{0.113} & \textbf{0.9}\%
\\ \cline {2-6} 
 & 60,000 & 0.114 & 0.110 & \textbf{0.106} & \textbf{3.6\%}
\\ \hline \hline

 \multirow{4}{*}{NYT} 
 & 40,000 & 0.214 & 0.208 & \textbf{0.195} & \textbf{6.3\%}
\\ \cline {2-6} 
 & 50,000 & 0.211 & 0.206 & \textbf{0.185} & \textbf{10.2\%}
\\ \cline {2-6} 
 & 60,000 & 0.205 & 0.200 & \textbf{0.194} & \textbf{3.0\%}
\\ \hline 
\end{tabular}
\caption{L1 reconstruction error on various semi-synthetic datasets. Last column is percent improvement over Recover-KL (best performing Recover algorithm).}
\label{tab:l1_all}
\end{table}

\begin{figure}[t!]
 \includegraphics[width=\textwidth]{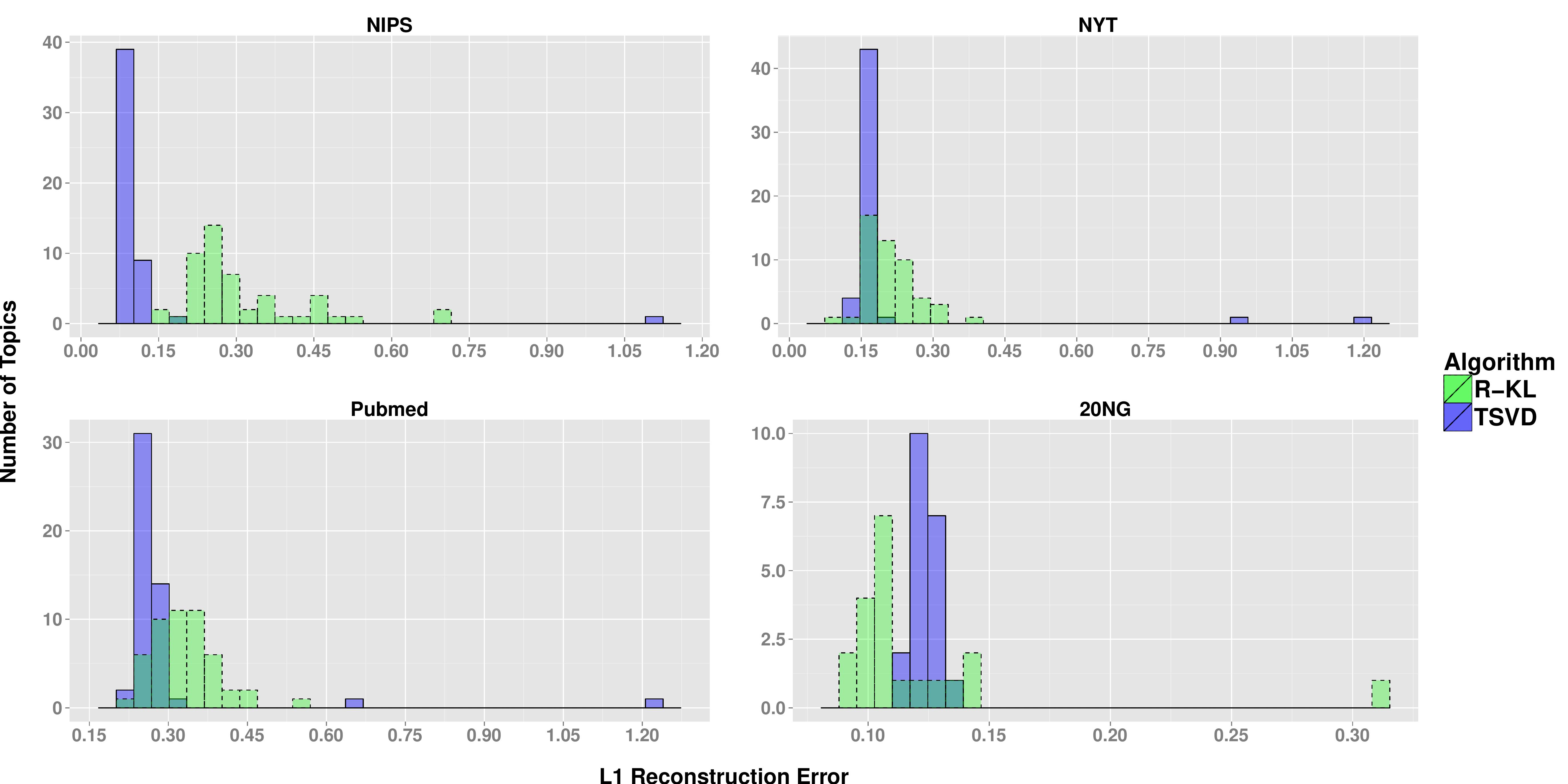}
 \caption{Histogram of $l_1$ error across topics for 40,000 synthetic documents. \emph{TSVD} (blue, solid border) gets better recovery on most topics ($>90\%$) for most datasets (leaving small number of outliers) than \emph{Recover-KL} (green, dashed border).}
 \label{fig:hist_all}
\end{figure}

\subsubsection{Topic Recovery on Real Data}
\paragraph{Perplexity:}
A standard quantitative measure used to compare topic models and inference algorithms is perplexity \citep{LDA}. Perplexity of a set of $D$ test documents, where each document $j$ consists of $m_j$ words, denoted by $\textbf{w}_j$, is defined as: 
$perp(D_{test}) = exp \left\{ - \frac{\sum_{j=1}^{D}\log p(\textbf{w}_j)}{\sum_{j=1}^{D}m_j} \right\}$.
To evaluate perplexity on real data, the held-out sets consist of 350 documents for NIPS, 10000 documents for NYT and Pubmed, and 6780 documents for 20NewsGroup.
Table \ref{tab:perp_tc} shows the results of perplexity on the 4 datasets.
TSVD gives comparable perplexity with Recover-KL, results being slightly better on NYT and 20NewsGroup which are larger datasets with moderately high mean document lengths. 

\paragraph{Topic Coherence:}

\cite{mimno} proposed Topic Coherence as a measure of semantic quality of the learned topics by approximating user experience of topic quality on top $d_0$ words of a topic.
Topic coherence is defined as 
$TC(d_0) = \sum_{i \le d_0} \sum_{j<i} \log \frac{D(w_i,w_j)+e}{D(w_j)}$,
where $D(w)$ is the document frequency of a word $w$, $D(w_i,w_j)$ is the document frequency of $w_i$ and $w_j$ together, and $e$ is a small constant.
We evaluate TC for the top 5 words of the recovered topic distributions and report the average and standard deviation across topics. 
TSVD gives comparable results on Topic Coherence (see Table \ref{tab:perp_tc}).

\paragraph{Topics on Real Data:}
Table \ref{tab:topics} shows the top 5 words of all 50 matched pair of topics on NYT dataset for TSVD, Recover-KL and Gibbs sampling.
Most of the topics recovered by TSVD are more closer to Gibbs sampling topics. 
Indeed, the total average $l_1$ error with topics from Gibbs sampling for topics from TSVD is 0.034, whereas for Recover-KL it is 0.047 (on the NYT dataset).

\begin{table}[t!]
 \centering
 \begin{tabular}{|c|c|c|c||c|c|c|}
  \hline
  \multirow{2}{*}{\bf Corpus} & \multicolumn{3}{c||}{\bf Perplexity} & \multicolumn{3}{c|}{\bf Topic Coherence} \\ \cline{2-7}
  & {\bf R-KL} & {\bf R-L2} & {\bf TSVD} & {\bf R-KL} & {\bf R-L2} & {\bf TSVD}\\ \hline
  NIPS & 754 & \textbf{749} & 835 & -86.4 \plmi 24.5 & -88.6 \plmi 22.7 & \textbf{-65.2 \plmi 29.4} \\ \hline  
  NYT & 1579 & 1685 & \textbf{1555} & -105.2 \plmi 25.0 & \textbf{-102.1 \plmi 28.2} & -107.6 \plmi 25.7 \\ \hline  
  Pubmed & \textbf{1188} & 1203 & 1307 & -94.0 \plmi 22.5 & -94.4 \plmi 22.5 & \textbf{-84.5 \plmi 28.7} \\ \hline 
  20NG & 2431 & 2565 & \textbf{2390} & -93.7 \plmi 13.6 & \textbf{-89.4 \plmi 20.7} & -90.4 \plmi 27.0 \\ \hline 
 \end{tabular}
 \caption{Perplexity and Topic Coherence. R-KL is Recover-KL, R-L2 is Recover-L2. Standard deviation for topic coherence across topics is also shown.}
 \label{tab:perp_tc}
\end{table}

{\bf Summary: }
We evaluated the proposed algorithm, \textbf{TSVD}, rigorously on multiple datasets with respect to the state of the art (Recover), following the evaluation methodology of \cite{arora}.
In Table \ref{tab:l1_all} we show that the L1 reconstruction error for the new algorithm is small and on average \textbf{19.6\%} better than the best results of the Recover algorithms \cite{arora}.
We also demonstrate that on real datasets the algorithm achieves comparable perplexity and topic coherence to Recover (Table \ref{tab:perp_tc}.
Moreover, we show on multiple real datasets that the algorithm assumptions are well justified in practice.

\section*{Conclusion}
Real world corpora often exhibits the property that
in every document
there is one topic dominantly present. A standard SVD based procedure will not
be able to detect these topics, however TSVD, a thresholded SVD based procedure, as suggested in this paper,
discovers these topics. 
While SVD is time-consuming, there have been a host of recent sampling-based approaches which
make SVD easier to apply to massive corpora which may be distributed among many servers. 
We believe that
apart from topic recovery, thresholded SVD can be applied even more broadly to similar problems, such as matrix factorization, and will be the basis for future
research. 

\LTcapwidth=\textwidth
\begin{longtable}{|>{\centering\arraybackslash}m{0.3\textwidth} | >{\centering\arraybackslash}m{0.3\textwidth} | >{\centering\arraybackslash}m{0.3\textwidth}|}
\caption{{Top 5 words of matched topic pairs for TSVD, Recover-KL and Gibbs sampling. Catchwords and anchor words in top 5 words are highlighted for TSVD and Recover-KL}} \\
\hline \textbf{TSVD} & \textbf{Recover-KL} & \textbf{Gibbs} 
\label{tab:topics} 
\endfirsthead
\caption{{Top 5 words of matched topic pairs for TSVD, Recover-KL and Gibbs sampling. Catchwords and anchor words in top 5 words are highlighted for TSVD and Recover-KL}} \\
\hline \textbf{TSVD} & \textbf{Recover-KL} & \textbf{Gibbs} \\ \hline
\endhead

\hline \multicolumn{3}{|r|}{{\emph{Continued on next page}}} \\ \hline
\endfoot

\hline \hline
\endlastfoot
\hline
\textbf{zzz\_elian} \textbf{zzz\_miami} \textbf{boy} \textbf{father} \textbf{zzz\_cuba}  & zzz\_elian boy zzz\_miami father family  & zzz\_elian zzz\_miami boy father zzz\_cuba \\ \hline 
\textbf{cup} \textbf{minutes} \textbf{add} \textbf{tablespoon} \textbf{oil}  & cup minutes \textbf{tablespoon} add oil  & cup minutes add tablespoon oil \\ \hline 
game team \textbf{yard} \textbf{zzz\_ram} season  & game team season play \textbf{zzz\_ram}  & team season game coach zzz\_nfl \\ \hline 
book find \textbf{british} sales \textbf{retailer}  & book find school woman women  & book find woman british school \\ \hline 
\textbf{run} \textbf{inning} \textbf{hit} season game  & run season game \textbf{inning} hit  & run season game hit inning \\ \hline 
\textbf{church} \textbf{zzz\_god} \textbf{religious} \textbf{jewish} \textbf{christian}  & \textbf{pope} church book jewish religious  & religious church jewish jew zzz\_god \\ \hline 
\textbf{patient} \textbf{drug} \textbf{doctor} \textbf{cancer} \textbf{medical}  & patient drug doctor percent found  & patient doctor drug medical cancer \\ \hline 
\textbf{music} \textbf{song} \textbf{album} \textbf{musical} \textbf{band}  & black reporter zzz\_new\_york zzz\_black show  & music song album band musical \\ \hline 
\textbf{computer} \textbf{software} system zzz\_microsoft company  & web www site \textbf{cookie} cookies  & computer system software technology mail \\ \hline 
\textbf{house} \textbf{dog} \textbf{water} \textbf{hair} look  & room show look home house  & room look water house hand \\ \hline 
\textbf{zzz\_china} trade zzz\_united\_states \textbf{nuclear} official  & zzz\_china \textbf{zzz\_taiwan} government trade zzz\_party  & zzz\_china zzz\_united\_states zzz\_u\_s zzz\_clinton zzz\_american \\ \hline 
\textbf{zzz\_russian} \textbf{war} \textbf{rebel} \textbf{troop} \textbf{military}  & zzz\_russian zzz\_russia war zzz\_vladimir\_putin rebel  & war military zzz\_russian soldier troop \\ \hline 
\textbf{officer} \textbf{police} \textbf{case} \textbf{lawyer} \textbf{trial}  & \textbf{zzz\_ray\_lewis} police case officer death  & police officer official case investigation \\ \hline 
\textbf{car} driver \textbf{wheel} race \textbf{vehicles}  & car driver truck system model  & car driver truck vehicle wheel \\ \hline 
\textbf{show} \textbf{network} \textbf{zzz\_abc} \textbf{zzz\_nbc} \textbf{viewer}  & \textbf{con} zzz\_mexico son federal mayor  & show television network series zzz\_abc \\ \hline 
\textbf{com} \textbf{question} \textbf{information} \textbf{zzz\_eastern} \textbf{sport}  & com information question zzz\_eastern sport  & com information daily question zzz\_eastern \\ \hline 
\textbf{book} author writer com \textbf{reader}  & \textbf{zzz\_john\_rocker} player team right braves  & book word writer author wrote \\ \hline 
\textbf{zzz\_al\_gore} zzz\_bill\_bradley campaign president democratic  & zzz\_al\_gore \textbf{zzz\_bill\_bradley} campaign president percent  & zzz\_al\_gore campaign zzz\_bill\_bradley president democratic \\ \hline 
\textbf{actor} film play movie character  & goal play team season game  & film movie award actor zzz\_oscar \\ \hline 
\textbf{school} \textbf{student} \textbf{teacher} \textbf{district} program  & school student program million children  & school student teacher program children \\ \hline 
\textbf{tax} \textbf{taxes} \textbf{cut} billion plan  & \textbf{zzz\_governor\_bush} tax campaign taxes plan  & tax plan billion million cut \\ \hline 
\textbf{percent} \textbf{stock} \textbf{market} \textbf{fund} \textbf{investor}  & million percent tax bond fund  & stock market percent fund investor \\ \hline 
team player season coach zzz\_nfl  & team season player coach \textbf{zzz\_cowboy}  & team player season coach league \\ \hline 
family home friend room school  & look gun game point shot  & family home father son friend \\ \hline 
\textbf{primary} zzz\_mccain voter zzz\_john\_mccain zzz\_bush  & \textbf{zzz\_john\_mccain} zzz\_george\_bush campaign republican voter  & zzz\_john\_mccain zzz\_george\_bush campaign zzz\_bush zzz\_mccain \\ \hline 
\textbf{zzz\_microsoft} \textbf{court} company case law  & \textbf{zzz\_microsoft} company computer system software  & zzz\_microsoft company window antitrust government \\ \hline 
\textbf{company} million percent \textbf{shares} \textbf{billion}  & million company stock percent \textbf{shares}  & company million companies business market \\ \hline 
\textbf{site} \textbf{web} \textbf{sites} com www  & web site zzz\_internet company com  & web site zzz\_internet online sites \\ \hline 
\textbf{scientist} \textbf{human} \textbf{cell} \textbf{study} \textbf{researcher}  & dog quick jump \textbf{altered} food  & plant human food product scientist \\ \hline 
\textbf{baby} \textbf{mom} percent home family  & \textbf{mate} women bird film idea  & women look com need telegram \\ \hline 
\textbf{point} game \textbf{half} shot team  & point game team season zzz\_laker  & game point team play season \\ \hline 
\textbf{zzz\_russia} \textbf{zzz\_vladimir\_putin} zzz\_russian \textbf{zzz\_boris\_yeltsin} \textbf{zzz\_moscow}  & zzz\_clinton government \textbf{zzz\_pakistan} zzz\_india zzz\_united\_states  & government political election zzz\_vladimir\_putin zzz\_russia \\ \hline 
com \textbf{zzz\_canada} www \textbf{fax} information  & \textbf{chocolate} food wine flavor buy  & www com hotel room tour \\ \hline 
\textbf{room} \textbf{restaurant} building \textbf{fish} \textbf{painting}  & zzz\_kosovo police \textbf{zzz\_serb} war official  & building town area resident million \\ \hline 
\textbf{loved} family show friend play  & film show movie music book  & film movie character play director \\ \hline 
\textbf{prices} percent \textbf{worker} oil price  & percent stock market economy prices  & percent prices economy market oil \\ \hline 
\textbf{million} \textbf{test} shares \textbf{air} \textbf{president}  & air wind snow \textbf{shower} weather  & water snow weather air scientist \\ \hline 
\textbf{zzz\_clinton} \textbf{flag} \textbf{official} \textbf{federal} \textbf{zzz\_white\_house}  & \textbf{zzz\_bradley} zzz\_al\_gore campaign zzz\_gore zzz\_clinton  & zzz\_clinton president gay mayor zzz\_rudolph\_giuliani \\ \hline 
\textbf{files} article computer art ball  & show film country right women  & art artist painting museum show \\ \hline 
\textbf{con} percent zzz\_mexico federal official  & official \textbf{zzz\_iraq} government zzz\_united\_states oil  & zzz\_mexico drug government zzz\_united\_states mexican \\ \hline 
\textbf{involving} book film case right  & \textbf{test} women study student found  & plane flight passenger pilot zzz\_boeing \\ \hline 
\textbf{zzz\_internet} \textbf{companies} company \textbf{business} \textbf{customer}  & company companies deal zzz\_internet \textbf{zzz\_time\_warner}  & media zzz\_time\_warner television newspaper cable \\ \hline 
zzz\_internet companies company business customer  & newspaper \textbf{zzz\_chronicle} zzz\_examiner zzz\_hearst million  & million money worker company pay \\ \hline 
\textbf{goal} \textbf{play} \textbf{games} \textbf{king} game  & \textbf{zzz\_tiger\_wood} shot tournament tour player  & zzz\_tiger\_wood tour tournament shot player \\ \hline 
\textbf{zzz\_american} zzz\_united\_states \textbf{zzz\_nato} \textbf{camp} war  & zzz\_israel \textbf{zzz\_lebanon} peace zzz\_syria israeli  & zzz\_israel peace palestinian talk israeli \\ \hline 
\textbf{team} season game player play  & team game point season player  & race won win fight team \\ \hline 
\textbf{reporter} \textbf{zzz\_earl\_caldwell} zzz\_black black look  & \textbf{corp} group list oil meeting  & black white zzz\_black hispanic reporter \\ \hline 
campaign \textbf{zzz\_republican} \textbf{republican} \textbf{zzz\_party} primary  & \textbf{zzz\_bush} zzz\_mccain campaign republican voter  & gun bill law zzz\_congress legislation \\ \hline 
\textbf{zzz\_bush} \textbf{zzz\_mccain} \textbf{campaign} primary republican  & flag black \textbf{zzz\_confederate} right group  & flag zzz\_confederate zzz\_south\_carolina black zzz\_south \\ \hline 
\textbf{zzz\_john\_mccain} campaign zzz\_george\_bush zzz\_bush republican  & official government case officer security  & court law case lawyer right \\ \hline 
\end{longtable}

\bibliographystyle{plain}

\begin{appendices}
\section{Line of Proof}\label{sec:lemmas}

We describe the Lemmas we prove to establish the result. The detailed proofs are in the Section \ref{sec:proof}.

\subsection{General Facts}
We start with a consequence of the no-local-minimum assumption. We use that assumption
solely through this Lemma.
\begin{lemma}\label{isoperimetry}
Let $p_i(\zeta,l)$ be as in (\ref{p-i-zeta-l}).
If for some $\zeta_0\in \{ 0,1,\ldots ,m\}$ and $\nu\geq 0$,
$\sum_{\zeta \geq \zeta_0}p_{i}(\zeta,l)\geq \nu$ and also $\sum_{\zeta \leq \zeta_0}p_i(\zeta,l)\geq \nu$
then, $p_i(\zeta_0,l)\geq \frac{\nu}{m}$.
\end{lemma}
Next, we state a technical Lemma which is used repeatedly. It states that for every $i,\zeta,l$,
the empirical probability that $A_{ij}=\zeta/m$ is close to the true probability. Unsurprisingly, we
prove it using H-C. But we will state a consequence in the form we need in the sequel.
\begin{lemma}\label{one-side-big-1}
Let $p_i(\zeta,l)$ and
$q_i(\zeta,l)$ be as in (\ref{p-i-zeta-l}) and (\ref{q-i-zeta-l}).
We have
$$\forall i,l,\zeta: \prob \left(\left|  p_i(\zeta,l)-q_i(\zeta,l)\right| \geq \frac{\varepsilon}{2} \sqrt{w_0}\sqrt{p_i(\zeta,l)} + \frac{\varepsilon ^2w_0}{2}\right)\quad\leq \; 2\exp(-\varepsilon^2sw_0/8).$$
From this it follows that with probability at least $1-2\exp(-\varepsilon ^2w_0s/8)$,
$$\frac{1}{2}q_i(\zeta,l)-\varepsilon^2 w_0\quad\leq \; p_i(\zeta,l)\quad\leq \; 2q_i(\zeta,l)+2\varepsilon^2 w_0.$$
\end{lemma}
\subsubsection{Properties of Thresholding}
Say that a threshold $\zeta_i$ ``splits'' $T_l^{(2)}$ if $T_l^{(2)}$ has a significant number of $j$
with $A_{ij}>\zeta_i/m$ and also a significant number of $j$ with $A_{ij}\leq \zeta_i/m$. Intuitively,
it would be desirable if no threshold splits any $T_l$, so that, in $\bB$, for each $i,l$, either
most $j\in T_l^{(2)}$ have $B_{ij}=0$ or most $j\in T_l^{(2)}$ have $B_{ij}=\sqrt{\zeta_i}$. We now prove
that this is indeed the case with proper bounds. We henceforth refer to the conclusion of the Lemma below
by the mnemonic ``no threshold splits any $T_l$''.
\begin{lemma}\label{one-side-big-2}({\bf No Threshold Splits any $T_l$})
For a fixed $i,l$, with probability at least $1-2\exp(-\varepsilon ^2w_0s/8)$, the following holds:
$$\text{Min}\; \left( \prob (A^{(2)}_{ij}\leq \frac{\zeta_i}{m}\; ; j\in T_l^{(2)}),\quad
 \prob (A^{(2)}_{ij} > \frac{\zeta_i}{m}\; ; j\in T_l^{(2)})\right) \quad \leq 4m\varepsilon w_0.$$
\end{lemma}

Let $\mu$ be a $d\times s$ matrix whose columns are
given by $$\forall j\in T_l^{(2)}\; ,\; \mu_{.,j}=E(B_{.,j} \; |\; j\in T_l).$$
$\mu$ 's columns corresponding to all $j\in T_l$ are the same.
The entries of the matrix $\mu$ are fixed (real numbers) once we have $\bA^{(1)}$ (and the thresholds $\zeta_i$ are determined).
Note: We have ``integrated out $W$'', i.e., $$\mu_{ij}= \int_{W_{\cdot,j}} \prob(W_{.,j}|j\in T_l) E(B_{ij} |W_{.,j}).$$
(So, think of $W_{\cdot,j}$ for $\bA^{(1)}$ 's columns being picked first from which $\zeta_i$ is calculated.
$W_{\cdot,j}$ for columns of $\bA^{(2)}$ are not yet picked until the $\zeta_i$ are determined.)
But $\mu_{ij}$ are random variables before we fix $\bA^{(1)}$.
The following Lemma is a direct consequence of ``no threshold splits any $T_l$''.
\begin{lemma}\label{define-R}
Let $\zeta_i'=\Max( \zeta_i, 8\ln (20/\varepsilon w_0))$. With probability at least $1-4kd\exp(-\varepsilon^2sw_0/8)$ (over the choice of $\bA^{(1)}$):
\begin{align}\label{1001}
\forall l, \forall j\in T_l, \forall i: &\mu_{ij}\leq \varepsilon_l\sqrt{\zeta_i'} \text{   OR   } \mu_{ij}\geq \sqrt{\zeta_i'}(1-\varepsilon_l)\nonumber\\
\forall l, \forall i, & \text{Var}(B_{ij})\leq 2\varepsilon_l \zeta_i',
\end{align}
where, $\varepsilon_l=4m\varepsilon w_0/w_l$.
\end{lemma}
So far, we have proved that for every $i$, the threshold does not split any $T_l$. But this is not sufficient in itself
to be able to cluster (and hence identify the $T_l$), since, for example, this alone does not rule out the extreme cases
that for most $j$ in every $T_l$, $A_{ij}^{(2)}$ is above the threshold (whence $\mu_{ij}\geq (1-\varepsilon _l)\sqrt{\zeta_l'}$
for almost all $j$) or for most $j$ in no $T_l$ is $A_{ij}^{(2)}$ above the
threshold, whence, $\mu_{ij}\leq\varepsilon_l\sqrt{\zeta_i'}$ for almost all $j$. Both these extreme cases would make us loose all the information about $T_l$ due to thresholding; this scenario and
milder versions of it have to be proven not to occur. We do this by considering how thresholds handle catchwords. Indeed we
will show that for a catchword $i\in S_l$, a $j\in T_l$ has $A_{ij}^{(2)}$ above the threshold and
a $j\notin T_l$ has $A_{ij}^{(2)}$ below the threshold. Both statements will only hold with high
probability, of course and using this, we prove that $\mu_{.,j}$ and $\mu_{.,j'}$ are not too close
for $j,j'$ in different $T_l$ 's.  For this, we need the following Lemmas.
\begin{lemma}\label{threshold-exists}
For $i\in S_l$, and $l'\not= l$, we have with $\eta_i = \big\lfloor\; M_{il}(\alpha+\beta+\rho)m/2\; \big\rfloor$,
\begin{equation*}
\prob(A_{ij}\leq \eta_i/m\; |\; j\in T_l)\leq \varepsilon w_0 /20, \;\;
\prob(A_{ij}\geq \eta_i/m\; |\; j\in T_{l'})\leq \varepsilon w_0/20.
\end{equation*}
\end{lemma}
\begin{lemma}\label{different-mus}
With probability at least $1-8mdk\exp(-\varepsilon^2 w_0s/8)$, we have~
$$ \for j\in T_{l},j' \notin T_l, \; |\mu_{\cdot,j}-\mu_{\cdot,j'}|^2\geq \frac{2}{9}\alpha p_0m. $$
\end{lemma}
\subsubsection{Proximity}
Next, we wish to show that clustering as in TSVD
identifies the dominant topics correctly for most documents,
i.e., that $R_l\approx T_l$ for all $l$. For this, we will use
a theorem from \citep{KK} [see also \cite{AS}] which in this context says:
\begin{theorem}
\label{proximity-KK}
If all but a $f$ fraction of the
the $B_{\cdot,j}$ satisfy the ``proximity condition'', then TSVD identifies the
dominant topic in all but $c_1 f$ fraction of the documents correctly after polynomial number of iterations.
\end{theorem}
To describe the proximity condition, first let $\sigma$ be the maximum over all directions $v$ of the square root
of the mean-squared
distance of $B_{.,j}$ to $\mu_{.,j}$, i.e.,
$$\sigma^2=\text{Max}_{\|v\|=1} \frac{1}{s} |v^T(\bB -\mu)|^2 = \frac{1}{s}\|\bB-\mu\|^2.$$
The parameter $\sigma$ should remind the reader of standard deviation, which is indeed what
this is, since $E({\bB}|T_1,T_2,\ldots ,T_l)=\mu$. Our random variables $B_{.,j}$ being $d-$ dimensional vectors, we take the
maximum standard deviation in any direction.

\begin{definition}\label{proximity}
$\bB$ is said to satisfy the
proximity condition with respect to $\mu$, if for each $l$ and each $j\in T_l$ and
and each $l'\not= l$ and $j'\in T_{l'}$, the projection of
$B_{.,j}$ onto the line joining $\mu_{.,j}$ and $\mu_{.,j'}$ is closer to $\mu_{.,j}$ by
at least $$\Delta= \frac{c_0k}{\sqrt{w_0}}\sigma,$$
than it is to $\mu_{.,j'}$. [Here, $c_0$ is a constant.]
\end{definition}

To prove proximity, we need to bound $\sigma$. This will be the task of the subsection \ref{sec:norm} which
relies heavily on Random Matrix Theory.

\section{Proofs of Correctness}\label{sec:proof}

We start by recalling the H\"offding-Chernoff (H-C) inequality in the form we use it.

\begin{lemma}\label{H-C}{\bf H\"offding-Chernoff}
If $X$ is the average of $r$ independent random variables with values in $[0,1]$ and $E(X)=\mu$, then,
for an $t>0$,
\begin{equation*}
\prob(X\geq \mu+t)\leq \exp\left( -\frac{ t^2r}{2(\mu+t)}\right)\;   ;  \;
\prob(X\leq \mu-t)\leq \exp\left( -\frac{t^2r}{2\mu}\right).
\end{equation*}
\end{lemma}
\begin{proof} (of Lemma \ref{isoperimetry})
Abbreviate $p_i(\cdot, l)$ by $f(\cdot)$. We claim that
either (i) $f(\zeta)\geq f(\zeta-1)\forall 1\leq \zeta\leq \zeta_0$ or (ii) $f(\zeta+1)\leq f(\zeta)
\forall m-1\geq \zeta \geq \zeta_0.$ To see this, note that if both (i) and (ii) fail, we have
$\zeta_1\leq \zeta_0$ and $\zeta_2\geq \zeta_0$ with $f(\zeta_1)-f(\zeta_1-1)<0<f(\zeta_2+1)-f(\zeta_2)$. But then
there has to be a local minimum of $f$ between $\zeta_1$ and $\zeta_2$.
If (i) holds, clearly, $f(\zeta_0)\geq f(\zeta)\forall \zeta<\zeta_0$ and so the lemma follows.
So, also if (ii) holds.
\end{proof}
%

%
\begin{proof} (of Lemma \ref{one-side-big-1})
Note that
$q_i(\zeta,l)=\frac{1}{s} \left| \{ j\in T_l: A_{ij}=\zeta/m\}\right|=\frac{1}{s}\sum_{j=1}^sX_j,$
where, $X_j$ is the indicator variable of $A_{ij}=\zeta/m\; \wedge \; j\in T_l$.
$\frac{1}{s}\sum_j E(X_j)=  p_i(\zeta,l)$ and we apply H-C with
$t=\frac{1}{2}\varepsilon \sqrt{w_0}\sqrt{p_i(\zeta,l)}+\frac{1}{2}\varepsilon^2w_0$ and $\mu=p_i(\zeta,l)$.
We have $\frac{t^2}{\mu+t}\geq \varepsilon^2w_0/4$, as is easily seen by calculating the roots of the
quadratic $t^2-\frac{1}{4}t\varepsilon^2w_0-\frac{1}{4} \varepsilon ^2w_0\mu=0$. Thus we get the claimed for
$T_l$. Note that the same proof applies for $T_l^{(1)}$ as well as $T_l^{(2)}$.

To prove the second assertion, let $a=q_i(\zeta,l)$ and $b=\sqrt {p_i(\zeta,l)}$, then, $b$ satisfies the quadratic inequalities:
$$b^2-\frac{1}{2}\varepsilon \sqrt{w_0}b-(a+\frac{1}{2}\varepsilon^2w_0)\leq 0\; ;\;
b^2+\frac{1}{2}\varepsilon \sqrt{w_0}b-(a-\frac{1}{2}\varepsilon ^2w_0)\geq 0.$$
By bounding the roots of these quadratics, it is easy to see the second assertion after some calculation.
\end{proof}

\begin{proof} (of Lemma \ref{one-side-big-2})
Note that $\zeta_i$ is a random variable which depends
only on $A^{(1)}$. So, for $j\in T_l^{(2)}$, $A_{ij}$ are independent of $\zeta_i$.
Now, if
$$\prob (A_{ij}\leq\frac{\zeta_i}{m}\; ; j\in T_l^{(2)})> 4m\varepsilon w_0\;
\text{  and  } \prob (A_{ij}>\frac{\zeta_i}{m}\; ; j\in T_l^{(2)})> 4m\varepsilon w_0,$$
by Lemma (\ref{isoperimetry}), we have
$$\prob (A_{ij}=\frac{\zeta_i}{m};j\in T_l^{(2)})> 4\varepsilon w_0 .$$
Since $\prob (A_{ij}=\zeta/m;j\in T_l^{(1)})=\prob (A_{ij}=\zeta/m;j\in T_l^{(2)})$ for all $\zeta$,
we also have
\begin{equation}\label{701}
\prob (A_{ij}=\frac{\zeta_i}{m};j\in T_l^{(1)})= p_i(\zeta_i,l)> 4\varepsilon w_0.
\end{equation}
Pay a failure probability of $2\exp(- \varepsilon^2sw_0/8)$ and assume
the conclusion of Lemma (\ref{one-side-big-1}) and we have:
$$\frac{1}{s}\left| \{ j\in T_l^{(1)}: A_{ij}=\frac{\zeta_i}{m}\}\right|=q_i(\zeta_i,l)
\geq p_i(\zeta_i,l)-\frac{\varepsilon}{2}\sqrt{w_0p_i(\zeta_i,l)}-\frac{\varepsilon^2}{2}w_0.$$
Now, it is easy to see that $p_i(\zeta,l)-\frac{\varepsilon}{2}\sqrt{ w_0p_i(\zeta,l)}$ increases as
$p_i(\zeta,l)$ increases subject to (\ref{701}). So,
$$p_i(\zeta,l)-\frac{\varepsilon }{2}\sqrt{w_0p_i(\zeta,l)}-\frac{\varepsilon^2}{2}w_0
>
(4\varepsilon -\varepsilon^{3/2}-\frac{1}{2}\varepsilon ^2)w_0\geq 3\varepsilon w_0,$$
contradicting the definition of $\zeta_i$ in the algorithm. This completes the proof of the Lemma.
\end{proof}
\begin{proof}(of Lemma \ref{define-R}):
After paying a failure probability of
$4kd\exp(-\varepsilon^2sw_0/8)$, assume no threshold splits any $T_l$.
[The factors of $k$ and $d$ come in because we are taking the union bound over all
words and all topics.]
Then,
\begin{align*}
\prob (A_{ij}^{(2)}\leq\frac{\zeta_i}{m}\; |\; j\in T_l^{(2)})&=\sum_{\zeta=0}^{\zeta_i} p_i(\zeta,l)/\prob(j\in T_l)
\leq 4m\varepsilon \frac{w_0}{w_l}\\
\text{  or   }\prob (A_{ij}^{(2)}> \frac{\zeta_i}{m}\; |\; j\in T_l^{(2)})& =\sum_{\zeta=\zeta_i+1}^mp_i(\zeta,l)/\prob(j\in T_l)\leq 4m\varepsilon\frac{w_0}{w_l}.
\end{align*}
Wlg, assume that $\prob (A_{ij}\leq \zeta_i/m\; |\; j\in T_l)\leq\varepsilon_l$. Then,
with probability, at least $1-\varepsilon_l$, $A_{ij}^{(2)}>\zeta_i/m$. Now, either $\zeta_i<8\ln (20/\varepsilon w_0)$ and all
$B_{ij},j\in T_l$ are zero and then $\mu_{ij}=0$, or $\zeta_i\geq 8\ln (20/\varepsilon w_0)$, whence, $E(B_{ij}|j\in T_l)\in [(1-\varepsilon_l)\sqrt{\zeta_i'},\sqrt{\zeta_i'}]$.
So,  $\mu_{ij}\geq (1-\varepsilon_l)\sqrt{\zeta'_i}$ and $\prob(B_{ij}= 0)
\leq \varepsilon_l$. So,
$$\text{Var}(B_{ij}^2|j\in T_l)\leq (\sqrt{\zeta_i'}-(1-\varepsilon_l)\sqrt{\zeta_i'})^2\prob(B_{ij}=\sqrt{\zeta_i'}|j\in T_l)
+(\sqrt{\zeta_i'}-0)^2
\prob(B_{ij}=0|j\in T_l)\leq 2\varepsilon _l\zeta_i'.$$ This proves the lemma in this case. The other case is symmetric.
\end{proof}
\begin{proof}(of Lemma \ref{threshold-exists})
Recall that $P_{ij}=\sum_lM_{il}W_{lj}$ is the probability of word $i$ in document $j$ conditioned on $\bW$.
Fix an $i\in S_l$. From the dominant topic assumption,
\begin{equation}\label{1700}
\forall j\in T_l, P_{ij} = \sum_{l_1}M_{il_1} W_{l_1,j}\geq M_{il}W_{lj}\geq M_{il}\alpha.
\end{equation}
The $P_{ij}$ are themselves random variables. Note that (\ref{1700}) holds with probability 1.
From Catchword assumption and (\ref{alpha-beta-rho}), we get that
$$M_{il}\alpha-(\eta_i/m)\geq M_{il}\alpha -M_{il}((\alpha+\beta+\rho)/2)\geq M_{il}\alpha\delta/2.$$
Now, we will apply H-C with $\mu-t=\eta_i/m$ and $\mu\geq M_{il}\alpha$ for the $m$ independent words in a
document. By Calculus, the probability bound
from H-C
of $\exp( -t^2w_ls/2\mu)=\exp( - (\mu-(\eta_i/m) ) / 2\mu)$
is highest subject to the constraints $\mu\geq M_{il}\alpha; \eta_i\leq mM_{il}(\alpha+\beta+\rho)/2$,
when $\mu=M_{il}\alpha$ and $\eta_i=mM_{il}(\alpha+\beta+\rho)/2$, whence, we get
$$\prob( A_{ij}\leq \eta_i/m\;|\; j\in T_l)\leq \exp(-M_{il}\alpha \delta^2m/8)\leq \varepsilon w_0/20,$$
using (\ref{def:catch}).
Now, we prove the second assertion of the Lemma.
\begin{align}\label{1701}
&\forall j\in T_{l'}, l'\not= l, \sum_{l_1}M_{il_1}W_{l_1,j}=M_{il}W_{lj}+\sum_{l_1\not= l}M_{il_1}W_{l_1,j}\nonumber\\
&\leq M_{il}W_{lj} +\left(\Max_{l_1\not= l}M_{il_1}\right) (1-W_{lj})\nonumber\\
&\leq M_{il}(\beta +\rho).
\end{align}
$$\frac{\eta_i}{m}-M_{il}(\beta+\rho)\geq \frac{M_{il}(\alpha+\beta+\rho)}{2}-M_{il}(\beta+\rho)-\frac{1}{m}\geq \frac{3M_{il}\alpha\delta}{8},$$
using (\ref{def:catch}) and (\ref{alpha-beta-rho}).
Applying the first inequality of Lemma (\ref{H-C})
with $\mu+t=\eta_i/m$ and $\mu\leq M_{il}(\beta+\rho)$ and again using (\ref{def:catch}),
$$\prob( A_{ij}\geq \eta_i/m\; |\; j\in T_{l'})\leq\exp\left( -9M_{il}\alpha\delta^2m/64\right)\leq \varepsilon w_0/20.$$
\end{proof}
\begin{lemma}\label{zeta-eta}
For $i\in S_l$,  $\prob(\zeta_i<\eta_i)\leq 3mke^{-\varepsilon^2sw_0/8}$,
with $\eta_i$ as defined in Lemma \ref{threshold-exists}.
\end{lemma}

\begin{proof}
Fix attention on $i\in S_l$. After paying the failure probability
of $3mke^{-\varepsilon^2sw_0/8}$, assume the conclusions of Lemma (\ref{one-side-big-1}) hold
for all $\zeta,l$.
It suffices to show that
$$\left| \{ j: A_{ij}^{(1)}> \eta_i/m\}\right|\geq \frac{w_0s}{2}\; \; ,\; \; \left| \{ j: A_{ij}^{(1)}=\frac{\eta_i}{m}\}\right|<3w_0\varepsilon s,$$
since, $\eta_i$ is an integer and $\zeta_i$ is the largest integer satisfying the inequalities.
The first statement follows from first assertion of Lemma \ref{threshold-exists}.
The second statement is slightly more complicated. Using both the first and second assertions of Lemma \ref{threshold-exists},
we get that for all $l'$ (including $l'=l$), we have
$$\prob(A_{ij}=\eta_i/m|j\in T_{l'}^{(1)})\leq \varepsilon w_0/20.$$
$$\left| \{ j\in T_{l'}^{(1)}:A_{ij}=\eta_i/m\}\right| \leq \varepsilon w_0w_{l'}s/20 +\frac{\varepsilon}{2} \sqrt{w_0/w_{l'}}\sqrt{\varepsilon w_0/20}w_{l'}s+\frac{\varepsilon^2w_0w_{l'}}{2}$$
$$\leq
\frac{\varepsilon w_0s}{8}\left(w_{l'}+\sqrt{\varepsilon w_{l'}}\right)+\frac{\varepsilon ^2w_0s}{2}.$$
Now, adding over all $l'$ and using $\sum_{l'}\sqrt {w_{l'}}\leq \sqrt k \sqrt{\sum_{l'}w_{l'}}=\sqrt k$, we get
$$\left| \{ j: A_{ij}^{(1)}=\eta_i/m\}\right| \leq \varepsilon w_os,$$ since $\varepsilon \leq 1/ k$.
\end{proof}
\begin{lemma}\label{Il}
Define $I_l=\{ i\in S_l: \zeta_i\geq \eta_i\}$.
With probability at least $1-8mdk\exp( -\varepsilon^2w_0s/8)$, we have for all $l$,
$$\sum_{i\in I_l}\zeta_i'\geq m\alpha p_0/2.$$
\end{lemma}
\begin{proof}
After paying the failure probability, we assume the conclusion of Lemma \ref{one-side-big-1} holds
for all $i,\zeta,l$.
Now, by Lemma \ref{zeta-eta}, we have (with ${\bf 1}$ denoting the indicator function)
$$E\left( \sum_{i\in S_l}M_{il}{\bf 1}(\zeta_i<\eta_i)\right)\leq 3mk\exp(-\varepsilon ^2sw_0/8)\sum_{i\in S_l}M_{il},$$
which using Markov inequality implies that with probability at least
 $1-6mk\exp(-\varepsilon ^2sw_0/8)$,
\begin{equation}\label{222}
 \sum_{i\in I_l}M_{il} \geq \frac{1}{2}\sum_{i\in S_l}M_{il}\geq p_0/2,
\end{equation}
using (\ref{402}).
Note that by (\ref{def:catch}), no catchword has $\zeta_i'$ set to zero.
So,
$$\sum_{i\in I_l}\zeta_i'=\sum_{i\in I_l} \zeta_i\geq\sum_{i\in I_l} \eta_i\geq \sum_{I_l} mM_{il}\alpha/2\geq \alpha p_0m/2.$$
\end{proof}
\begin{proof}(of Lemma \ref{different-mus})
For this proof, $i$ will denote an element of $I_l$. By Lemma \ref{threshold-exists},
\begin{equation}\label{1709}
\forall i\in I_l, l'\not= l, \prob( A_{ij}>\frac{\zeta_i}{m}|j\in T_{l'}^{(1)})\leq \prob( A_{ij}> \eta_i/m|j\in T_{l'}^{(1)})\leq\frac{\varepsilon w_0}{20}.
\end{equation}
This implies by Lemma \ref{one-side-big-1}, for $l'\not= l$,
\begin{equation}\label{177}
\left| \{ j\in T_{l'}^{(1)}: A_{ij}>\frac{\zeta_i}{m}\}\right| \leq w_{l'}s\left( \frac{\varepsilon w_0}{20}+\varepsilon \sqrt{w_0/w_{l'}}\sqrt{\varepsilon w_0}/4\right)+w_0\varepsilon ^2s/2.
\end{equation}
Summing over all $l'\not=l$, we get (using $\sum_{l'}\sqrt{w_{l'}}\leq \sqrt{\sum w_{l'}}\sqrt k\leq 1/\sqrt\varepsilon$ by (\ref{varepsilon-inequality}))
$$\sum_{l'\not= l}\left| \{ j\in T_{l'}^{(1)}: A_{ij}>\frac{\zeta_i}{m}\}\right|\leq \varepsilon w_0s.$$
Now the definition of $\zeta_i$ in the algorithm implies that:
$$\sum_{\zeta>\zeta_i} q_i(\zeta,l)=\left| \{ j\in T_l^{(1)}: A_{ij}>\frac{\zeta_i}{m}\}\right| \geq \left( \frac{w_0}{2}-\varepsilon w_0\right)s\geq w_0s/4.$$
So, by Lemma \ref{one-side-big-1},
\begin{align*}
&\prob(j\in T_l;A_{ij}>\zeta_i/m)=\sum_{\zeta>\zeta_i} p_i(\zeta,l)\geq \frac{1}{2}\sum_{\zeta>\zeta_i}q_i(\zeta,l)-\varepsilon^2w_0m\\
&\geq \frac{w_0}{8}-\varepsilon^2w_0m\geq w_0/9,
\end{align*}
using (\ref{varepsilon-inequality}).
Next let $\tilde p=  \prob (A_{ij}=\zeta_i/m;j\in T_l)$. Since $|\{ j\in T_l^{(1)}:A_{ij}=\zeta_i/m\}|\leq 3\varepsilon w_0s$,
by the definition of $\zeta_i$ in the algorithm,
we get by a similar argument
\begin{equation}\label{1703}
 \tilde p\leq 2q_i(\zeta_i,l)+2\varepsilon^2w_0\leq  7\varepsilon w_0.
\end{equation}
Now, by Lemma \ref{isoperimetry}, we have
$$\tilde p\geq \text{Min} \left( \frac{w_0}{9m}\; ,\; \frac{1}{m}\prob (A_{ij}\leq\zeta_i/m;j\in T_l^{(2)})\right).$$
By (\ref{varepsilon-inequality}), $7\varepsilon w_0<w_0/9m$ and so we get:
$$\prob (A_{ij}\leq\zeta_i/m;j\in T_l^{(2)})\leq 7\varepsilon mw_0.$$
Noting that by (\ref{def:catch}),
no catchword has $\zeta_i'$ set to zero,
 $\prob (B_{ij}=0|j\in T_l ^{(2)})\leq 7\varepsilon mw_0/w_l\leq 1/6$,
by (\ref{varepsilon-inequality}).
This implies
$$\mu_{ij}\geq \frac{5}{6}\sqrt{\zeta_i'}.$$
Now, by (\ref{1709}), we have for $j'\notin T_l$,
$$\mu_{ij'}\leq \sqrt {\zeta_i'}/6.$$
So, we have
$$\left| \mu_{\cdot,j}-\mu_{\cdot,j'}\right|^2\geq \sum_{i\in I_l}(\mu_{ij}-\mu_{ij'})^2\geq (4/9)\sum_{i\in I_l}\zeta_i'.$$
Now Lemma (\ref{Il}) implies the current Lemma.
\end{proof}
\subsection{Bounding the Spectral norm}\label{sec:norm}

\begin{theorem}\label{spectral-norm}
Fix an $l$. For $j\in T_l$, let $R_{.,j}=B_{.,j}-\mu_{.,j}$. [The $R_{.,j}, j\in T_l$ are vector-valued random
variables which are independent, even conditioned on the partition $T_1,T_2,\ldots ,T_k$.]
 With probability at least
$1-10mdk\exp(-\varepsilon^2 w_0s/8)$, we have
$||R||^2\leq c k w_0\varepsilon sm^2.$
Thus,
$$||\bB - \mu ||^2\leq c\varepsilon w_0sm^2k^2.$$
\end{theorem}
We will apply Random Matrix Theory, in particular the following theorem,
to prove
Theorem \ref{spectral-norm}.
\begin{theorem}\label{vershynin} \citep[Theorem 5.44] {vers}\label{thm:||U||}
Suppose $R$ is a $d\times r$ matrix with columns $R_{\cdot, j}$ which are
independent identical vector-valued random variables. Let $U=E(R_{\cdot, j}R_{\cdot ,j}^T)$
be the inertial matrix of $R_{\cdot ,j}$. Suppose $|R_{\cdot ,j}|\leq\nu$ always. Then, for
any $t>0$, with probability at least $1-de^{-ct^2}$, we have\footnote{$||R||$ denotes the spectral norm of $R$.}
$$||R||\leq ||U||^{1/2}\sqrt r+t\nu.$$
\end{theorem}
We need the following Lemma first.
\begin{lemma}\label{sum-zeta-prime}
With probability at least $1-\exp(-s\varepsilon w_0/3)$, we have
\begin{equation}\label{zeta0}
\zeta_0\leq 4m\lambda\;\; ;\;\; \sum_{i}\zeta_i'\leq 4km
\end{equation}
\end{lemma}

\begin{proof}
The probability of word $i$ in document $j$, is given by:
$P_{ij}=\sum_lM_{il}W_{lj}\leq \lambda_i$ (where, $\lambda_i =\max_lM_{il}$).
If $\lambda_i<\frac{1}{m}\ln (20/\varepsilon w_0)$, then, $\prob (A_{ij}>(8/m)\ln (20/\varepsilon w_0))\leq \varepsilon w_0$
by H-C (since $A_{ij}$ is the average of $m$ i.i.d. trials).
Let $X_j$ be the indicator function of $A_{ij}>(8/m)\ln (20/\varepsilon w_0)$.
$X_j$ are independent and so using H-C, we see that with probability at least
$1-\exp(-\varepsilon w_0s/3)$, less than $w_0s/2$ of the
$A_{ij}$ are greater $(8/m)\ln(20/\varepsilon w_0)$, whence,
$\zeta_i'=0$.
So we have (using the union bound over all words):
$$\prob \left( \sum_{i:\lambda_i <(1/m)\ln (20/\varepsilon w_0)}\zeta_i'>0\right)\leq d\exp(-\varepsilon w_0s/3).$$
If $\lambda_i\geq (1/m)\ln (20/\varepsilon w_0)$, then
$$\prob (A_{ij}>4\lambda_i)\leq e^{-\lambda_i m}\leq \varepsilon w_0/2,$$\
which implies by the same $X_j$ kind of argument that with probability at least
$1-\exp(-\varepsilon w_0s/4)$, for a fixed $i$, $\zeta_i\leq 4\lambda_im$.
Using the union bound over all words and adding all $i$, we get that with probability
at least $1-2d\exp(-\varepsilon w_0s/4)$,
$$\sum_i \zeta_i'\leq 4m\sum_i\lambda_i\leq 4m\sum_{i,l}M_{il}\leq 4km.$$
Now we prove the bound on $\zeta_0$. For each fixed $i,j$, we have
$\prob(A_{ij}\geq 4\lambda)\leq e^{-m\lambda}\leq \varepsilon w_0$. Now,
let $Y_j$ be the indicator variable of $A_{ij}\geq 4\lambda$. The $Y_j,j=1,2,\ldots ,s$
are independent (for each fixed $i$). So, $\prob(\zeta_i\geq 4m\lambda)\leq
\prob (\sum_j Y_j\geq w_0s/2)\leq e^{-\varepsilon w_os/3}$. Using an union
bound over all words, we get that $\prob(\zeta_0>4m\lambda)\leq de^{-\varepsilon w_0/3}$
by H-C.
\end{proof}
\begin{proof} (of Theorem \ref{spectral-norm})
First,  $$||U||=\Max_{|v|=1}E(v^TR_{\cdot,j})^2\leq E(|R_{\cdot, j}|^2)\leq
2\varepsilon _l\sum_i\zeta_i'\leq 8\varepsilon_lkm,$$
by Lemma (\ref{sum-zeta-prime}) and Lemma (\ref{define-R}).
We can also take $\nu=2\sqrt {km}$ in Theorem \ref{thm:||U||} and
with $t=\sqrt{\varepsilon mw_0s}$, the first statement of the current theorem follows
(noting $r=w_ls$). The second statement follows by just paying a factor of $k$ for the
$k$ topics.
\end{proof}
\subsection{Proving Proximity}
From Theorem (\ref{spectral-norm}), the $\sigma$ in definition \ref{proximity} is $\sqrt{c\varepsilon w_0m^2k^2}$.
So, the $\Delta$ in definition \ref{proximity} is $cc_0\sqrt{\varepsilon}k^2m$. So it suffices to prove:

\begin{lemma}\label{proximity-holds}
For $j\in T_l$ and $j'\in T_{l'}, l'\not= l$, let $\hat B_{.,j}$ be the projection of $B_{.,j}$
onto the line joining $\mu_{.,j}$ and $\mu_{.,j'}$. The probability that
$|\hat B_{.,j}-\mu_{.,j'}|\leq |\hat B_{.,j}-\mu_{.,j}| + cc_0k^{2}\sqrt{\varepsilon }m$ is at most
$c\varepsilon mw_0 \sqrt k/\sqrt{\alpha p_0}$. Hence, with probability at least $1-cmdk\exp(-cw_0\varepsilon^2s)$,
the number of $j$ for which $B_{.,j}$ does not satisfy the proximity condition is
at most $c\varepsilon_0w_0\delta s/10c_1$.
\end{lemma}
\begin{proof}
After paying the failure probability of $cmdk\exp(-w_0s\varepsilon^2/8)$,
of Lemmas
(\ref{sum-zeta-prime}) and (\ref{different-mus}),
assume that $\zeta_0\leq 4m\lambda$ , $|\mu_{.,j}-\mu_{.,j'}|^2\geq 2\alpha mp_0/9$ and $\sum_i\zeta_i'\leq 4km$.

Let $X=(B_{\cdot, j}-\mu_{\cdot,j})\cdot (\mu_{\cdot ,j'}-\mu_{\cdot ,j})$.
$X$ is a random variable, whose expectation is 0 conditioned on $j\in T_l^{(2)}$.

Since $\prob (B_{ij}=\sqrt{\zeta_i'}|j\in T_l)=\mu_{ij}/\sqrt{\zeta_i'}$, we have:
\begin{align*}
E|X|&\leq E\sum_i |B_{ij}-\mu_{ij}| \; |\mu_{ij'}-\mu_{ij}|\\
&= \sum_i \left[ (\sqrt{\zeta_i'}-\mu_{ij})\frac{\mu_{ij}}{\sqrt{\zeta_i'}}+(1-\frac{\mu_{ij}}{\sqrt{\zeta_i'}})\mu_{ij}\right]
|\mu_{ij}-\mu_{ij'}|\\
&\leq 2\varepsilon _l \sum_i\sqrt{\zeta_i'}|\mu_{ij}-\mu_{ij'}|\quad\text{ by Lemma \ref{define-R}}\\
&\leq 2\varepsilon _l \left(\sum _i\zeta_i'\right)^{1/2} |\mu_{.,j}-\mu_{.,j'}|\leq 4\varepsilon_l\sqrt{km}|\mu_{.,j}-\mu_{.,j'}|.
\end{align*}
Now apply Markov inequality to get
$$\prob(|X|\geq \frac{1}{8}|\mu_{.,j}-\mu_{.,j'}|^2)\leq 32 \varepsilon_l\sqrt{km}/|\mu_{.,j}-\mu_{.,j'}|\leq 80\varepsilon _l\sqrt{k/\alpha p_0}.$$
If $|X|\leq |\mu_{.,j}-\mu_{.,j'}|^2/8$, then, $|\hat B_{.,j}-\mu_{.,j'}|\geq |\hat B_{.,j}-\mu_{.,j}|+3|\mu_{.,j}-\mu_{.,j'}|/4\geq
|\hat B_{.,j}-\mu_{.,j}|+cc_0k^{2}\sqrt\varepsilon m$, by (\ref{varepsilon-inequality}). This proves the first
assertion of the Lemma.

The second statement of Lemma follows
by applying H-C to the random variable $\sum_jZ_j/s$, where, $Z_j$ is the indicator random variable of $B_{.,j}$ not
satisfying the proximity condition (and using (\ref{varepsilon-inequality}).)
\end{proof}

The last Lemma implies that the algorithm TSVD correctly identifies the dominant topic in all but
at most $\varepsilon_0w_0/10$ fraction of the documents by Theorem (\ref{proximity-KK}).
\begin{lemma}\label{dom-topic-correct}
With probability at least $1-\exp(-w_0s\varepsilon^2/8)$, TSVD correctly identifies the dominant topic
in all but at most $\varepsilon_0w_0\delta /10$ fraction of documents in each $T_l$.
\end{lemma}

\subsection{Identifying Catchwords}

Recall the definition of $J_l$ from Step 5a of the algorithm. The two lemmas below are roughly converses of each
other which prove roughly that $J_l$ consists of those $i$ for which $M_{il}$ is strictly higher than $M_{il'}$. Using
them, Lemma \ref{pure-docs} says that almost all the $\varepsilon_0w_0s/2$ documents found in Step 6 of the algorithm
are $1-\delta$ pure for topic $l$.
\begin{lemma}\label{iinJl}
Let $\nu =\gamma (1-2\delta)/(1+\delta)$.
If $i\in J_l$, then for all $l'\not= l$, $M_{il}\geq\nu M_{il'}$ and $M_{il}\geq \frac{3}{m\delta^2}\ln (20/\varepsilon w_0)$.
\end{lemma}
\begin{proof}
It is easy to check that the assumptions (\ref{delta-inequality}) and
(\ref{alpha-beta-rho})imply $\nu\geq 2$.
Let $i\in J_l$. By the definition of $J_l$ in the algorithm,
$g(i,l)\geq (4/m\delta^2)\ln (20/\varepsilon w_0).$ Note that $P_{ij}\leq \Max_{l_1}M_{il_1}$ for all $j$.
So,
\begin{equation}\label{555}
\max_{l_1}M_{il_1}\geq \frac{3}{m\delta^2}\ln (20/\varepsilon w_0).
\end{equation}
If the Lemma is false, then, for $l'$ attaining Max$_{l_1\not= l}M_{il_1}$, we have $ M_{il}<\nu M_{il'}$.
Recall $R_{l'}$ defined in Step 4c of the algorithm.
Let $$\hat T_{l'}=R_{l'}\cap \; (\text{ the set of $1-\delta$ pure documents in }T_{l'}).$$
Since all but $\varepsilon_0w_0s/10$ documents in $T_{l'}$ belong to $R_{l'}$, we have
$|\hat T_{l'}|\geq 0.9\varepsilon_0 w_0s$.
For $j\in \hat T_{l'}$,
$P_{ij}\geq M_{il'}W_{l'j}\geq (1-\delta)M_{il'} $. So,
$\prob (A_{ij}<M_{il'}(1-2\delta ))\leq \exp(-m\delta^2M_{il'}/3)\leq \varepsilon w_0/4$ using (\ref{555}).
Thus the number of documents in $R_{l'}$ for which $A_{ij}\geq M_{il'}(1-2\delta)$ is at least
$0.9\varepsilon_0w_0 s-3\varepsilon w_0s\geq  .6\varepsilon_0w_0s$.
This implies that
with probability at least $1-\exp(-c\varepsilon^2sw_0)$,  $g(i,l')\geq M_{il'}(1-2\delta)$.

Now, for $j\in T_l$, $P_{ij} \leq \Max(M_{il},M_{il'})\leq \nu M_{il'}$. So,
$\prob (A_{ij}>M_{il'} \nu(1+\delta))\leq \varepsilon w_0/4$, again using (\ref{555}).
At most $\varepsilon_0 w_0s/10$ documents of other $T_{l_1}$, $l_1\not= l$ are in $R_l$
(by Lemma \ref{dom-topic-correct}).
So, whp,
$g(i,l)\leq M_{il'}\nu(1+\delta)$ and so we have
$$g(i,l)\leq \frac{\nu(1+\delta)}{1-2\delta}g(i,l'),$$
contradicting the definition of $J_l$.
So, we must have that $M_{il}\geq\nu M_{il'}$ for all $l'\not= l$.
The second assertion of the Lemma now follows from (\ref{555}).
\end{proof}

\begin{lemma}\label{iinJl1}
If $M_{il}\geq \Max \left( \frac{5}{m\delta^2}\ln (20/\varepsilon w_0),\Max_{l'\not= l} \frac{1}{\rho}\; M_{il'}\right)$,
then, with probability at least $1-\exp(-c\varepsilon^2w_0s)$, we have that $i\in J_l$. So, $S_l\subseteq J_l$.
\end{lemma}

\begin{proof}
Let $\hat T_l=R_l\cap $ (set of $1-\delta$ pure documents in $T_l$).
For $j\in \hat T_l$, $P_{ij}\geq M_{il}(1-\delta)$ which implies that whp,
(since $|\hat T_l|\geq 0.9\varepsilon _0 s$, again by Lemma \ref{dom-topic-correct})
\begin{equation}\label{g-l-high}
g(i,l)\geq M_{il}(1-2\delta)
\end{equation}
On the other hand, for $j\in T_{l'}$ and for $l'\not= l$, $i:M_{il'}\leq\rho M_{il}$ (hypothesis of the Lemma),
$P_{ij}\leq M_{il}W_{lj}+\rho M_{il}(1-W_{lj})\leq M_{il}(\beta+\rho)$. So whp,
\begin{equation}\label{g-lprime-low}
g(i,l')\leq M_{il}(\beta+\rho)(1+\delta).
\end{equation}
From (\ref{g-l-high}) and (\ref{g-lprime-low}) and hypothesis of the Lemma, it follows that
$$g(i,l) \geq \Max \left( \frac{4}{m\delta^2}\ln (1/\varepsilon w_0),\frac{(1-2\delta)}{(1+\delta)(\beta+\rho)}\; g(i,l')\right).$$
So, $i\in J_l$ as claimed. It only remains to check that $i$  in $ S_l$ satisfies the hypothesis of the Lemma which is
obvious.
\end{proof}
\begin{lemma}\label{pure-docs}
Let $\nu_l=\sum_{i\in J_l}M_{il}$ and let $L$ be the set
of $\lfloor (s\varepsilon_0w_0/2)\rfloor$ $A_{.,j}$ 's whose average is returned in Step 6 of the TSVD Algorithm as $\hat M_{.,l}$. With probability at least $1-c\exp (-c \varepsilon^2w_0s)$, we have:
 \begin{equation}
 \left| \frac{1}{|L|}\sum_{j\in L}(A_{.,j}-M_{.,l})\right|_1\leq O(\delta)\label{902}.
\end{equation}
\end{lemma}

\begin{proof}
The proof needs care since $J_l$ is itself a random set dependent on $A^{(2)}$. To understand the
proof intuitively, if we pretend that there is no conditioning of $J_l$ on $A^{(2)}$, then, basically,
our arguments in Lemma \ref{iinJl} would yield this Lemma. However, we have to work harder to avoid
conditioning effects. Define
$$K_l=\{ i: M_{il}\geq \nu M_{il'}\forall l'\not= l; M_{il}\geq (3/m\delta^2)\ln (20/\varepsilon w_0)\}.$$
Note that $K_l$ is not a random set; it does not depend on $A$, just on $M$ which is fixed. Lemma \ref{iinJl}
proved that $J_l\subseteq K_l$. Since $\sum_iM_{il}=1$, we have $|K_l|\leq m\delta^2/3$. The probability bounds
given here will be after conditioning on $\bW$. [In other words, we prove statements of the form
$\prob ({\cal E}|\bW )\leq a$ which is (the usual) shorthand for:
for each possible value $w$ of the matrix $W$, $\prob ({\cal E} \; |\; \bW =w)\leq a$.] This will be possible, since,
even after fixing $W$, the $A_{.,j}$ are independent, though certainly not identically distributed now, since the $W_{.,j}$ may
differ.

%
%
%

For $i\in K_l$, we have for all $j$, $P_{ij}=\sum_{l'}M_{il'}W_{l'j}\leq M_{il}$, since,
$M_{il'}\leq M_{il}/\nu\leq M_{il}/2$ for $l'\not= l$. For any $x\leq M_{il}$,
$$\prob ( |A_{ij}^{(2)}-P_{ij}|\geq\delta M_{il}\; |\; W,P_{ij}=x)\leq 2\exp\left( -\frac{\delta^2M_{il}^2m}{2(1+\delta)x}\right)\leq 2\exp\left(-
\frac{m\delta^2M_{il}}{3}\right).$$
Noting that $m\delta^2M_{il}\geq 3\ln (20/\varepsilon w_0)$ for $i\in K_l$, we get
$$\prob ( |A_{ij}^{(2)}-P_{ij}|\geq\delta M_{il}\; |\; W) \leq \varepsilon w_0/20.$$
Using the union bound over all $i\in K_l$ yields (for each $j\in [s]$),
$$\prob ( \exists i\in K_l: |A_{ij}^{(2)}-P_{ij}|\geq\delta M_{il}\; |\; W)\leq \frac{m\delta^2\varepsilon w_0}{20}\leq \frac{\varepsilon_0w_0\delta^2}{20},$$
by (\ref{varepsilon-inequality}).
Let $$BAD=\{ j: \exists i\in K_l: |A_{ij}^{(2)}-P_{ij}|\geq \delta M_{il}\}.$$
Using the independence of $A_{.,j}$, (even conditioned on $W$), apply H-C
to get that for the event
\begin{align}\label{event-bad}
{\cal E}&: |BAD|\geq \frac{s\varepsilon_0w_0\delta}{10}\nonumber\\
\prob({\cal E}\; |\; W)&\leq 2\exp(-c\varepsilon w_0s).
\end{align}
After paying the failure probability, for the rest of the proof, assume that
$\neg {\cal E}$ holds.
Let $U_l=\{ j: W_{lj}\geq 1-\delta\}$. By the dominant topic assumption, we know that
$|U_l|\geq \varepsilon_0w_0s$.
So, $|U_l\setminus BAD|\geq 4\varepsilon_0w_0s/5$
and we get (using (\ref{varepsilon-inequality})):
\begin{equation}
\forall N_l\subseteq K_l, \left| \{ j: W_{lj}\geq 1-\delta \; ;\;  \sum_{i\in N_l}A_{ij}^{(2)}\geq (1-2\delta )\sum_{i\in N_l}M_{il}\}\right| \geq 4\varepsilon_0w_0s/5\label{900}.
\end{equation}
Now consider $j:W_{lj}\leq (1-6\delta)$ and $i\in K_l$.
$$P_{ij}\leq M_{il}W_{lj}+\sum_{l'\not= l}M_{il'}W_{l'j}
\leq M_{il}(1-6\delta)+\frac{M_{il}}{\nu}6\delta\leq M_{il}(1-3\delta) ,$$
since by (\ref{delta-inequality}) and (\ref{alpha-beta-rho}), we have that $\nu\geq 2$.
So, for a $j$  with $W_{lj}\leq 1-6\delta$ to have $\sum_{i\in J_l}A_{ij}^{(2)}\geq (1-2\delta)\nu_l$,
$j$ must be in $BAD$. This gives us
\begin{equation}
\forall N_l\subseteq K_l,
\left| \{ j: W_{lj}\leq (1-6\delta)\; ;\; \sum_{i\in N_l}A_{ij}^{(2)}\geq (1-2\delta )\sum_{i\in N_l}M_{il}\}\right| \leq \varepsilon_0w_0\delta s /10\label{901}.
\end{equation}

Let $L$ be the set of $\lfloor \varepsilon_0w_0s/2\rfloor $ $j$ achieving the highest
$\sum_{i\in J_l}A_{ij}^{(2)}$. By the above,
$L$ contains at most $\varepsilon _0\delta s/5$ $j$'s with $W_{lj}<1-6\delta$, the rest being
$j$ with $W_{lj}\geq 1-6\delta$.
%
So are we finished with the proof - i.e., does this prove (\ref{902})? The answer is unfortunately, no.
We can show from the above that $\sum_{i\in J_l}|A_{ij}-M_{il}|\leq O(\delta)$
for most $j\in  L$ and so the average of $A_{.,j},j\in L$ is close to $M_{.,l}$
when we restrict only to $i\in J_l$.
But, on words not in $J_l$, we have
not proved that the average of $A_{ij}^{(2)},j\in  L$ is close to $M_{.,l}$. We will do so presently,
but first note that this is not a trivial task. For example, if say, $M_{il}=\Omega(1/d)$ for all
$i\notin K_l$ (or for a fraction of them) so that $\sum_{i\notin K_l}M_{il}\in\Omega (1)$, then an individual $A_{.,j}$
could have $O(m)$ of the $A_{ij},i\notin K_l$ set to $1/m$. [One copy of each of $O(m)$ words picked to be in the
document.] But then we would have $|A_{.,j}-M_{.,l}|_1\in \Omega(1)$ which is too much error. We will show that
since we are taking the average over $L$ and not just a single document, this will not happen. But the
proof is again tricky because of conditioning: both $J_l$ and $L$ depend on the data. So,
to argue that the average over $L$ behaves well, we have to prove it for each possible $L$.
There are at most ${s\choose \lfloor (\varepsilon_0w_0s/2)\rfloor}\leq (2/\varepsilon _0w_0s)^{\varepsilon_0w_0s/2}$
possible $L$ 's and we will be able to take the union bound over all of them.
\begin{claim}
With probability at least $1-cmdk\exp (-c\varepsilon ^2w_0s)$, we have
for each $L\subseteq [s]$ with $|L|=\lfloor (\varepsilon_0w_0s/2)\rfloor$:
$$\left| \frac{1}{|L|}\sum_{j\in L} (A_{\cdot, j}-P_{\cdot ,j})\right|_1\leq O(\delta).$$
\end{claim}
\begin{proof}
Let $X=\left| \frac{1}{|L|}\sum_{j\in L} (A_{\cdot, j}-P_{\cdot ,j})\right|_1$. Each $A_{\cdot, j}$
is itself the average of $m$ independent choices of words. So
$$X=\left| \frac{1}{m|L|}\sum_{j\in L}\sum_{r=1}^m (A_{\cdot, j}^{(r)}-P_{\cdot ,j})\right|_1.$$
So, $X$ is a function of $m|L|$ independent random variables. Changing any one of these arbitrarily changes
$X$ by at most $1/m|L|$.

Recall the Bounded Difference inequality \cite{McDiarmid89}:
\begin{lemma}
Let $z_1,\ldots,z_n,z_i'$ are ~$(n+1)$~independent random variables each taking values
in ${\mathcal Z}$ and $h$ be a measurable function from ${\mathcal Z}^n$ to $\RR$ with constants $r_i \ge 0, i \in [n]$ such that
$$ max_{z_1,\ldots,z_n,z_i' \in {\mathcal Z}} |h(z_1,\ldots,z_n) - h(z_1,\ldots,z_i',\ldots,z_n)| \le r_i$$
If $E(h)$ is the expectation of $h$ then
$ \prob\left(|h - E(h)|\ge t \right) \le 2 \exp\left(-\frac{t^2}{\sum_{i=1}^n r_i^2}\right).$
\end{lemma}
\vspace{-0.5em}
Using this we get
$$\prob (|X-EX|\geq c\delta)\leq \exp(-c\delta^2\varepsilon_0w_0sm).$$
The ``extra'' $m$ in the exponent helps kill the upper bound
of $(2/\varepsilon _0w_0s)^{\varepsilon_0w_0s/2}$ on the number of $L$ 's and gives us
$$|X-EX|\leq O(\delta) \forall L.$$
We still have to bound $EX$.
By Jenson's inequality,
\vspace{-0.5em}
$$EX\leq \frac{1}{|L|}\sum_i \left( E\left( (\sum_{j\in L}(A_{ij}-P_{ij}))\; ^2\right)\right)^{1/2}\leq
   \frac{1}{|L|}\sum_i \sqrt{\sum_{j\in L_l}P_{ij}}\leq \sqrt d/\sqrt {|L|},$$
   where, we have used the independence of $A_{\cdot,j}$ and the fact that $E(A_{ij}-P_{ij})^2=\text{Var}(A_{ij})$.
This proves the claim.
\end{proof}
\vspace{-1em}
We now bound
$ \left| \frac{1}{|L|}\sum_{j\in L}(P_{.,j}-M_{.,l})\right|_1$.
Note that by (\ref{900}) and (\ref{901}),
all but at most $\varepsilon_0w_0\delta s/10$ of the $j$ 's in $L$ have $W_{lj}\geq 1-6\delta$, whence,
we get $|P_{.,j}-M_{,l}|_1\leq 6\delta$ for these $j$. For the $j$ with $W_{lj}<1-6\delta$, we just use
$|P_{.,j}-M_{.,l}|_1\leq 2$. So
$$\left| \frac{1}{|L|}\sum_{j\in L}(P_{.,j}-M_{.,l})\right|_1\leq 6\delta + \frac{0.2\varepsilon_0w_0\delta s}{10|L_l|}
\in O(\delta).$$

This finishes the proof of (\ref{902}).

\end{proof}

\end{appendices}

\end{document}